\documentclass[runningheads]{llncs}

\PassOptionsToPackage{pagebackref,breaklinks,colorlinks,citecolor=eccvblue}{hyperref}

\usepackage[mobile]{eccv}

\usepackage{eccvabbrv}
\usepackage{graphicx}
\usepackage{booktabs}
\usepackage[accsupp]{axessibility}
\usepackage{hyperref}
\usepackage{orcidlink}
\usepackage{amsmath}
\usepackage{amssymb}
\usepackage{placeins}
\usepackage{float}
\usepackage{wrapfig}
\usepackage{caption}
\usepackage{siunitx}  
\usepackage[table]{xcolor}
\usepackage{tabularx}
\usepackage{multirow}
\usepackage{algorithm}
\usepackage{algpseudocode}
\usepackage{bbm}
\usepackage{xcolor}

\definecolor{AstroStart}{HTML}{7A28CB} 
\definecolor{AstroEnd}{HTML}{00D2FF}

\newcommand{\AstroGradient}{%
\textcolor{AstroStart!100!AstroEnd}{A}%
\textcolor{AstroStart!87!AstroEnd}{s}%
\textcolor{AstroStart!75!AstroEnd}{t}%
\textcolor{AstroStart!62!AstroEnd}{r}%
\textcolor{AstroStart!50!AstroEnd}{o}%
\textcolor{AstroStart!37!AstroEnd}{l}%
\textcolor{AstroStart!25!AstroEnd}{a}%
\textcolor{AstroStart!12!AstroEnd}{b}%
\textcolor{AstroStart!0!AstroEnd}{e}%
}

\definecolor{ourscolor}{HTML}{E6F2FF}

\begin{document}

\title{%
  \texorpdfstring{\AstroGradient}{Astrolabe}%
  : Steering Forward-Process Reinforcement Learning for%
  \texorpdfstring{\\}{ }%
  Distilled Autoregressive Video Models%
}

\titlerunning{\textbf{\AstroGradient}}


\author{Songchun Zhang\inst{1}\and Zeyue Xue\inst{2,3}\and Siming Fu\inst{2}\and Jie Huang\inst{2} \and Xianghao Kong\inst{1} \and Yue-Ma\inst{1} \and Haoyang Huang\inst{2} \and Nan Duan\inst{2\S} \and Anyi Rao\inst{1\S}}

\authorrunning{Zhang. et al.}

\institute{$^1$HKUST $^2$JD Explore Academy $^3$HKU\\
\href{https://franklinz233.github.io/projects/astrolabe/}{Project Page}
}
\maketitle

\begin{figure*}[h]
  \centering
  \vspace{-0.8cm}
  \includegraphics[width=\textwidth]{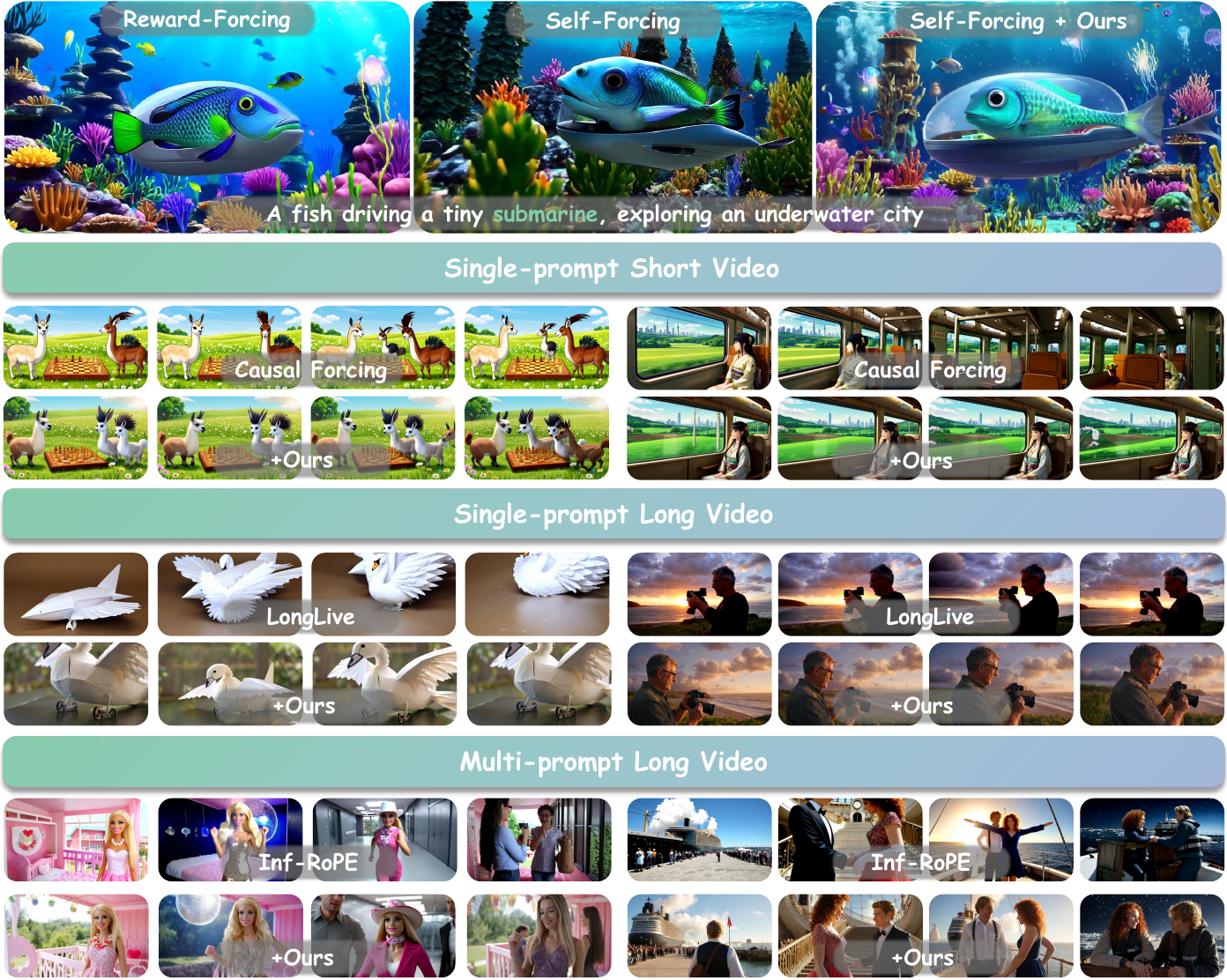}
  \caption{\AstroGradient~efficiently aligns distilled streaming video models with human preferences without re-distillation, enhancing baselines (e.g., Causal Forcing~\cite{zhu2026causal}, LongLive~\cite{yang2025longlive} and Infinite-RoPE~\cite{yesiltepe2025infinityrope}) by mitigating artifacts and improving temporal consistency. We demonstrate boosted perceptual quality across: (Top) single-prompt short, (Middle) single-prompt long, and (Bottom) multi-prompt long video generation.}
  \vspace{-0.1cm}
  \label{fig:teaser}
\end{figure*}

\vspace{-2em}

\FloatBarrier
\begin{abstract}
Distilled autoregressive (AR) video models enable efficient streaming generation but frequently misalign with human visual preferences. Existing reinforcement learning (RL) frameworks are not naturally suited to these architectures, typically requiring either expensive re-distillation or solver-coupled reverse-process optimization that introduces considerable memory and computational overhead.
We present Astrolabe, an efficient online RL framework tailored for distilled AR models. To overcome existing bottlenecks, we introduce a forward-process RL formulation based on negative-aware fine-tuning. By contrasting positive and negative samples directly at inference endpoints, this approach establishes an implicit policy improvement direction without requiring reverse-process unrolling. To scale this alignment to long videos, we propose a streaming training scheme that generates sequences progressively via a rolling KV-cache, applying RL updates exclusively to local clip windows while conditioning on prior context to ensure long-range coherence. Finally, to mitigate reward hacking, we integrate a multi-reward objective stabilized by uncertainty-aware selective regularization and dynamic reference updates. Extensive experiments demonstrate that our method consistently enhances generation quality across multiple distilled AR video models, serving as a robust and scalable alignment solution.

\end{abstract}

\section{Introduction}
\label{sec:intro}

Recent advances in diffusion models~\cite{he2022latent, ho2022video, blattmann2023align, blattmann2023stable, chen2023videocrafter1, gupta2024photorealistic, zhao2024identifying, xing2024dynamicrafter, zhao2025controlvideo,wan2025wan,zheng2024open} have enabled unprecedented quality in video synthesis, yet deploying these systems for real-time interactive applications remains challenging.
Conventional video diffusion models rely on extensive multi-step denoising processes, resulting in prohibitive generation latencies.
Furthermore, the bidirectional attention mechanism employed by most architectures processes all frames jointly, precluding streaming generation, wherein frames must be produced sequentially.
These constraints have motivated a paradigm shift toward efficient, autoregressive alternatives.

To overcome these constraints, several distilled autoregressive models~\cite{huang2025self, yang2025longlive, cui2025self,shin2025motionstream,huang2025live, hong2025relic} have emerged.
These methods distill pretrained bidirectional video diffusion models into efficient autoregressive models via distribution matching distillation (DMD)~\cite{yin2024one}.
The resulting models leverage KV-caching for streaming inference, enabling real-time generation with the potential to support long video generation.
%
However, while distillation ensures the student mimics the teacher's distribution, it lacks optimization for human preference. Consequently, the generated outputs frequently exhibit artifacts and unnatural motion dynamics, remaining misaligned with human preferences.

Concurrently, online RL has demonstrated high efficacy in aligning LLMs with human preferences~\cite{ouyang2022training, guo2025deepseek}.
This success motivates a natural question: \textit{can online RL be applied to align distilled streaming video models with human visual expectations without reverting to computationally expensive pre-training or re-distillation pipelines?}
Aligning these models via existing methods introduces non-trivial challenges.
Previous attempts at reward-guided distillation~\cite{lu2025rewardforcing} merely bias the supervised distillation loss by prioritizing samples with higher rewards. 
While this shifts the output distribution toward high-reward regions, it lacks a mechanism for active exploration and fails to penalize suboptimal generation samples.
On the other hand, applying online RL via reverse-process optimization~\cite{xue2025dancegrpo,liu2025flowgrpo} requires log-probability estimation along the sampling trajectory. 
This couples the algorithm to specific solvers and necessitates storing intermediate trajectory states, adding substantial memory and computational overhead that erodes the efficiency advantages of streaming models.

We present \textbf{Astrolabe}, an efficient and stable online RL framework for distilled AR video models, as shown in Figure~\ref{fig:pipeline}. 
Firstly, 
to bypass the limitations of reward-weighted distillation and the overhead of reverse-process RL, we introduce a trajectory-free alignment strategy tailored for distilled AR video generation. 
Drawing on the principles of negative-aware fine-tuning~\cite{zheng2025diffusionnft}, our approach contrasts positive and negative generations to establish an implicit policy improvement direction. 
Requiring only clean inference endpoints, our method sidesteps solver-specific unrolling and full trajectory storage, better preserving the efficiency inherent to streaming architectures.
Then, while this resolves per-clip alignment efficiently, scaling to long videos remains challenging: naively unrolling and backpropagating through extended sequences is prohibitively expensive.
To address this, we introduce a streaming training scheme that generates videos progressively while applying RL updates only to short segments, conditioning on prior context to retain long-range coherence.
Furthermore, to prevent models from reward hacking at the expense of overall aesthetics, the framework employs a multi-reward formulation covering visual quality, motion dynamics, and text alignment.
This optimization process is further stabilized by an uncertainty-aware selective regularization strategy that restricts KL penalties to samples lacking auxiliary consensus, alongside a dynamic reference update mechanism that accommodates shifting distributions during online learning.

Extensive experiments on various distilled AR models validate the effectiveness of our method. 
Figure~\ref{fig:teaser} showcases a diverse set of representative results, demonstrating that the proposed framework consistently enhances generation quality across different settings. 
Comprehensive evaluations demonstrate improvements across multiple benchmarks.
In summary, the primary contributions of our work are as follows: (1) Astrolabe, an online reinforcement learning framework formulated to align distilled streaming video models with human visual preferences; (2) a streaming training scheme that enables scalable alignment of long videos via segment-wise optimization under historical context; and (3) a suite of stabilization techniques, encompassing multi-reward optimization and dynamic regularization, to mitigate reward hacking.
\section{Related Work}
\label{sec:related}

\subsection{Video Generative Models}
Diffusion models achieve remarkable success in video synthesis~\cite{he2022latent, ho2022video, singer2022make, blattmann2023align, blattmann2023stable, chen2023videocrafter1, gupta2024photorealistic, zhao2024identifying, xing2024dynamicrafter, zhao2025controlvideo, zhao2022egsde}. The strong scalability of Diffusion Transformers (DiTs)~\cite{bao2023all, peebles2023scalable} facilitates the emergence of large-scale models~\cite{yang2024cogvideox, bao2024vidu, kong2024hunyuanvideo, wan2025wan} that generate high-quality content by jointly denoising all frames. However, this full-sequence generation requires simultaneous processing of all frames, which incurs substantial latency and precludes real-time interaction. Consequently, autoregressive approaches~\cite{wu2021godiva, hong2022cogvideo, wu2022nuwa, weissenborn2019scaling, yan2021videogpt, zhao2025ultravico, zhao2025riflex, deng2024autoregressive, kondratyuk2023videopoet} emerge to enable streaming generation by producing frames sequentially.

\subsection{Autoregressive Video Generation}
To circumvent the limitation of bidirectional diffusion models, autoregressive (AR) approaches enable streaming generation by producing frames sequentially.
While AR models are inherently suitable for real-time applications, early methods~\cite{hu2024acdit, gao2024ca2} relying on Teacher Forcing (TF) suffer from severe error accumulation during long-video synthesis. 
Recent studies explore novel training paradigms to resolve this train-test misalignment. 
Diffusion Forcing~\cite{chen2024diffusion} introduces conditioning at arbitrary noise levels, while CausVid~\cite{yin2025slow} employs block causal attention and distills bidirectional teacher via DMD~\cite{yin2024one}. 
More recently, Self-Forcing~\cite{huang2025self} and its successors~\cite{lu2025rewardforcing,yang2025longlive,yesiltepe2025infinityrope,cui2025self,guo2025end} establish post-training frameworks that systematically mitigate error accumulation.
Identifying an architectural gap in the initial ODE distillation phase of these frameworks, Causal Forcing~\cite{zhu2026causal} reveals that distilling from a bidirectional teacher violates frame-level injectivity. 
By employing an AR teacher for initialization instead, it theoretically bridges this gap to achieve superior real-time generation.

\subsection{Reinforcement Learning for Generative Models}
Recent successes in large language models~\cite{guo2025deepseek,ouyang2022training} highlight the efficacy of on-policy reinforcement learning via memory-efficient algorithms like GRPO~\cite{shao2024deepseekmath}. 
For diffusion models, DiffusionDPO~\cite{wallace2024diffusion} utilizes off-policy pairs, while Dance-GRPO~\cite{xue2025dancegrpo} and Flow-GRPO~\cite{liu2025flowgrpo} perform alignment by estimating reverse-trajectory log-probabilities. 
These reverse-process methods inherently couple the training objective to specific solvers and demand full trajectory storage. 
To bypass this, DiffusionNFT~\cite{zheng2025diffusionnft} introduces solver-agnostic forward-process policy optimization. 
Building on this, WorldCompass~\cite{wang2026worldcompass} recently adapted NFT to autoregressive world models~\cite{sun2025worldplay}.
However, their framework directly optimizes heavy pre-distilled teacher models. Extending RL to highly efficient distilled AR video models remains an open problem.

\begin{figure*}[t!]
    \centering
    \includegraphics[width=1.0\linewidth]{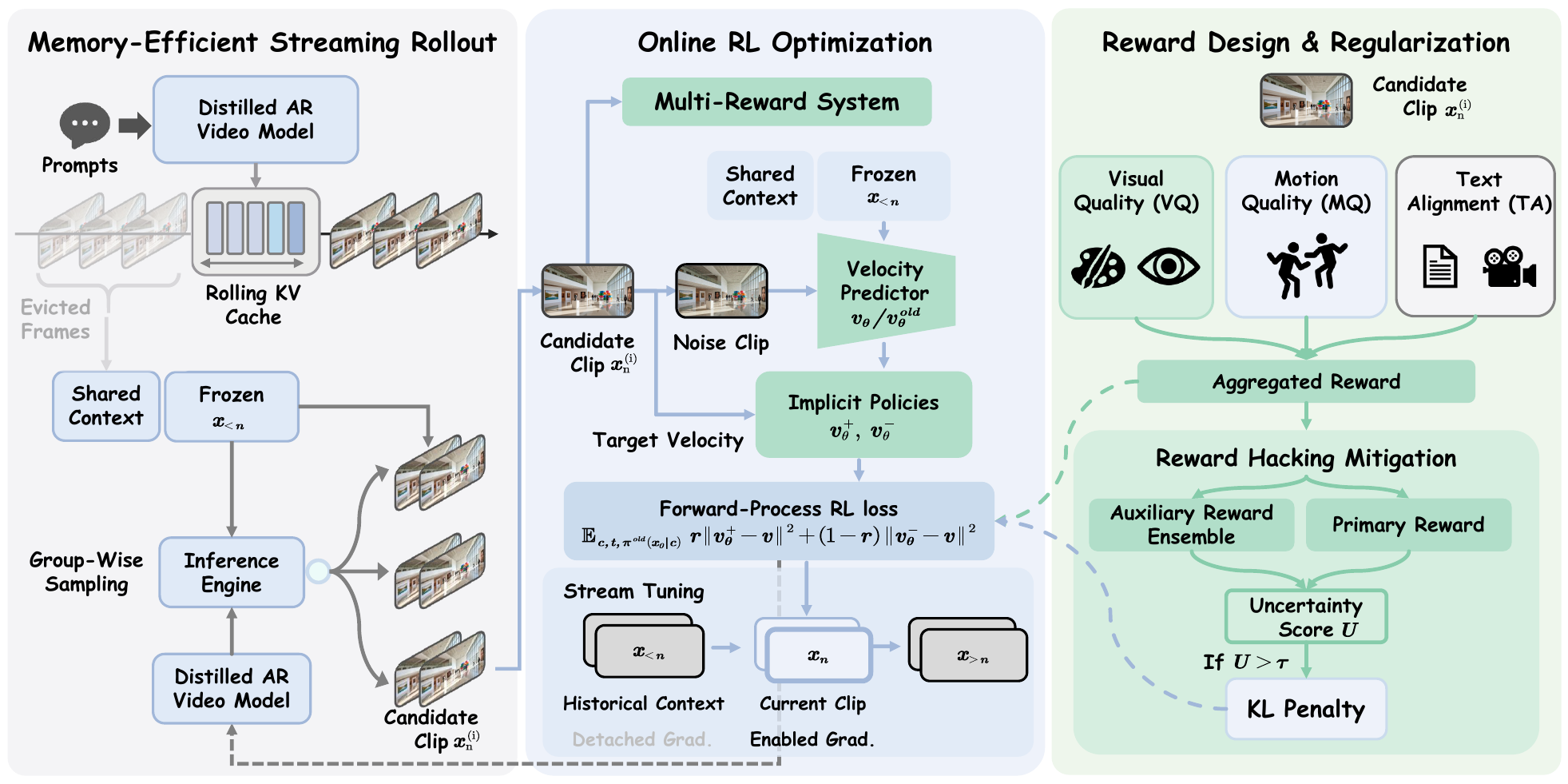}
    \vspace{-4mm}
    \caption{\textbf{Overview of Astrolabe.} We propose a memory-efficient RL framework for distilled streaming video models. The method combines group-wise streaming rollout using a rolling KV cache for efficient group-wise sampling (see left), and clip-level forward-process RL for solver-agnostic optimization (see middle). To scale to long videos, we utilize Streaming Long Tuning with detached historical gradients. Furthermore, a multi-reward formulation paired with uncertainty-based selective regularization is employed to effectively mitigate reward hacking during training (see right). The pseudocode of the algorithm can be found in the supplementary materials.}
\vspace{-4mm}
\label{fig:pipeline}
\end{figure*}

\section{Methodology}
\label{sec:method}

Given a distilled autoregressive video diffusion model optimized for real-time generation, our goal is to further align it with human preferences through online reinforcement learning in the post-training stage. We propose Astrolabe, a memory-efficient framework combining streaming rollout with forward-process RL optimization. 
Section~\ref{sec:preliminaries} reviews the foundations of AR video diffusion and forward-process RL. 
Section~\ref{sec:rollout} details our memory-efficient streaming rollout mechanism for scalable exploration. 
Section~\ref{sec:optimization} presents the online RL optimization strategy, encompassing clip-level forward-process RL and streaming long tuning. 
Finally, Section~\ref{sec:reward} formulates our multi-reward design and selective regularization approach to mitigate reward hacking.

\subsection{Preliminaries}
\label{sec:preliminaries}

\noindent{\textbf{Autoregressive Video Diffusion Models.}}
AR video model factorizes the joint distribution as $p(x_{1:N}) = \prod_{i=1}^N p(x_i | x_{<i})$. 
Following the flow matching formulation, each conditional $p(x_i|x_{<i})$ is modeled by defining a probability path $x_i^t = (1-t) x_i + t \epsilon_i$, where $\epsilon_i \sim \mathcal{N}(0,I)$ and $t \in [0,1]$. 
The model predicts the velocity field $v_\theta$ conditioned on text $c$ and the KV cache of preceding frames. 
Training paradigms such as Teacher Forcing (TF) and Diffusion Forcing (DF) minimize the frame-wise MSE between the predicted and true targets. 
In TF, timesteps $t$ are shared across frames with clean ground-truth context $x_{<i}$, whereas in DF, independent timesteps $t_i$ are sampled for each frame using noisy context $x_{j<i}^{t_j}$.
Both suffer from exposure bias due to the mismatch between training context and inference-time generation. 
To mitigate this, Self-Forcing~\cite{huang2025self} employs autoregressive rollouts $\{x_{1:N}^\theta\} \sim \prod_{i=1}^N p_\theta(x_i|x_{<i})$ to simulate inference dynamics. 
The objective aligns the velocity predictions of the model on these self-generated trajectories with the scores provided by teacher model.

\noindent{\textbf{Forward-Process Reinforcement Learning.}}
To avoid the likelihood estimation challenges of reverse-process RL, DiffusionNFT~\cite{zheng2025diffusionnft} optimizes diffusion models by applying rewards directly to the forward process. Given a clean generated sample $x$ with a normalized reward $\tilde{r} \in [0,1]$, a noisy version $x^t$ is constructed for timestep $t \in [0, 1]$. Using the current ($v_\theta$) and old ($v_{\theta_{\text{old}}}$) velocity predictors, implicit positive and negative policies are defined via interpolation:
\begin{equation}
    v^+ = (1-\beta) v_{\theta_{\text{old}}} + \beta v_\theta, \quad v^- = (1+\beta) v_{\theta_{\text{old}}} - \beta v_\theta
\label{eq:interpolation}
\end{equation}
where $\beta$ controls the interpolation strength. 
The policy loss contrasts these implicit policies against the target forward velocity $v_{target}$:
\begin{equation}
    \mathcal{L}_{\text{policy}} = \tilde{r} \|v^+ - v_{target}\|_2^2 + (1-\tilde{r}) \|v^- - v_{target}\|_2^2
\label{eq:policy_loss}
\end{equation}
This trajectory-free formulation requires only clean generated samples, enabling highly efficient, solver-agnostic training.

\subsection{Memory-Efficient Streaming Rollout}
\label{sec:rollout}

Standard RL paradigms rely on sequence-level rollouts with global rewards. For autoregressive (AR) video generation, this introduces two critical bottlenecks: the temporal credit assignment problem, where sparse global scores fail to isolate localized visual degradation, and the prohibitive memory overhead of maintaining independent KV caches for long sequences. To overcome these limitations, we propose a group-wise streaming rollout strategy.

\noindent{\textbf{Rolling KV Cache with Frame Sinks.}}
We maintain a rolling KV cache to bound memory usage. Let the sequence of generated clips be denoted as $x_{1}, x_{2}, \dots, x_{N}$. At generation step $n$, naïvely caching the full history $x_{<n}$ incurs a KV memory cost that grows linearly with video length, quickly becoming prohibitive for long-horizon rollouts. To resolve this, we construct a restricted visual context window $\mathcal{C}_n$ comprising two components: a \emph{frame sink} of $S$ permanently retained frames that anchors global semantic context to prevent long-range drift, and a rolling window of the $L$ most recent frames that provides fine-grained local conditioning. The model attends exclusively to the KV cache of $\mathcal{C}_n$ to generate the next clip $x_n \sim \pi_\theta(\cdot | \mathcal{C}_n, c)$. Since $S$ and $L$ are fixed hyperparameters independent of total video length $N$, the resident KV memory remains constant regardless of how long the video grows, enabling real-time streaming rollout.

\noindent{\textbf{Clip-level Group-wise Sampling.}}
Rather than generating $G$ independent long trajectories from scratch, we autoregressively sample the visual history exactly once and freeze its KV cache as a shared prefix. At the $n$-th step, utilizing the memory-efficient KV states of $\mathcal{C}_n$, the model decodes $G$ independent candidate clips in parallel:
\begin{equation}
    x_{n}^{(i)} \sim \pi_\theta(\cdot | \mathcal{C}_n, c), \quad \text{for } i \in \{1, \dots, G\}
\end{equation}
This clip-level rollout restricts the generation overhead to the local chunk rather than the full sequence. By sharing the frozen context prefix across all $G$ candidates, the additional cost of group-wise sampling is incurred only once per step rather than once per trajectory, substantially reducing rollout time and eliminating out-of-memory bottlenecks during reinforcement learning.

\subsection{Online RL Optimization}
\label{sec:optimization}

\noindent{\textbf{Clip-level Forward-Process RL.}}
For each candidate $x_n^{(i)}$, we evaluate a composite reward $R(x_n^{(i)}, c)$ and compute its advantage $A^{(i)}$ via group-wise mean-centering:
\begin{equation}
A^{(i)} = R(x_n^{(i)}, c) - \frac{1}{G}\sum_{j=1}^G R(x_n^{(j)}, c)
\end{equation}
This advantage is then normalized as $\tilde{r}_i = \text{clip}(A^{(i)} / A_{\max}) / 2 + 0.5$.
For our $T=4$ distilled model, the timestep $t$ is sampled from $\mathcal{T}_{distill}$. Crucially, we discard the adaptive loss weighting of DiffusionNFT~\cite{zheng2025diffusionnft}, as it triggers gradient explosion under large discretization gaps in distilled AR settings.
Conditioned on text $c$ and the shared KV cache $\mathcal{C}_n$, we construct the noised sample $x_n^{t,(i)}$ to predict velocities $v_\theta$ and $v_{\theta_{\text{old}}}$. The model is optimized directly via the implicit policy loss $\mathcal{L}_{\text{policy}}$ (Eq.~\ref{eq:policy_loss}) by substituting $x_n^{t,(i)}$ to derive $v_{target}$. To further mitigate reward hacking, this objective is complemented by an uncertainty-aware selective KL penalty (Section~\ref{sec:reward}).

\noindent{\textbf{Streaming Long Tuning.}}
Distilled AR models suffer from a train-short/test-long mismatch, where accumulated prediction errors cause inevitable long-horizon degradation. 
To address this, our training paradigm strictly simulates the dynamics of long-sequence inference while decoupling the forward rollout from gradient computation.
Specifically, we first perform a full forward pass to accumulate the KV cache up to the target step. 
Upon reaching the active training window $x_n$, the KV cache of all preceding frames $x_{<n}$ is explicitly detached from the computation graph. 
This detached cache serves as historical context, mimicking the progressively noisy conditions encountered during autoregressive generation.
Gradients are then backpropagated through the active window. 
This exact formulation inherently bounds the training memory usage, circumventing the cost of backpropagation through extended trajectories. 


\subsection{Reward Design and Regularization}
\label{sec:reward}

\noindent{\textbf{Multi-reward Formulation.}}
Scalar reward functions obscure specific quality dimensions and often inadvertently encourage the model to exploit one attribute over others. 
To address this, we formulate a composite reward integrating three distinct axes: Visual Quality (VQ), Motion Quality (MQ), and Text-Video Alignment (TA). 
We compute the Visual Quality (VQ) reward as the mean HPSv3~\cite{ma2025hpsv3} score over the top 30\% of frames. 
Excluding lower-scoring frames prevents transient motion blur from disproportionately penalizing the overall aesthetic assessment.
For the Motion Quality (MQ) reward, we evaluate temporal consistency using a pre-trained VideoAlign~\cite{videoalign} strictly on grayscale inputs; removing color forces the metric to focus on motion dynamics rather than texture. 
Finally, the Text Alignment (TA) reward employs the standard RGB VideoAlign to measure the semantic correspondence between the text and the generated video

\noindent{\textbf{Uncertainty-Aware Penalty.}}
To prevent uniform KL regularization from indiscriminately suppressing high-quality generations, we introduce a selective KL penalty targeting reward hacking via reward rank disagreement~\cite{he2025gardo}. 
For each candidate $x_n^{(i)}$, we quantify sample uncertainty as the rank discrepancy between the primary reward model $p$ and $M-1$ auxiliary models: $\Delta_{\text{rank}}^{(i)} = \text{rank}_p^{(i)} - \frac{1}{M-1}\sum_{m \neq p} \text{rank}_m^{(i)}$. 
High positive values indicate likely reward hacking lacking ensemble consensus. We mask these risky samples using $\mathcal{M}^{(i)} = \mathbbm{1}[\Delta_{\text{rank}}^{(i)} > \tau]$, where $\tau$ is the $(1-\rho)$-th percentile of positive discrepancies (with risk ratio $\rho$). The total objective $\mathcal{L} = \mathcal{L}_{\text{policy}} + \lambda_{\text{KL}} \mathcal{L}_{\text{KL}}$ applies the KL penalty strictly to masked samples, preserving optimization flexibility for clean data. 
Furthermore, to mitigate distributional shifts during online RL, the policy $\theta_{\text{old}}$ follows an EMA update, and the reference policy conditionally resets ($\theta_{\text{ref}} \leftarrow \theta$) when policy deviation surpasses $\tau_{\text{KL}}$ or epochs reach $K_{\max}$.

\section{Experiments}
\label{sec:exp}

\subsection{Experimental Setup}
\noindent{\textbf{Implementation Details.}}
To validate the effectiveness of our method, we evaluate Astrolabe on distilled autoregressive models. 
We adopt base models trained via Self-Forcing~\cite{huang2025self}, Causal-Forcing~\cite{zhu2026causal}, and LongLive~\cite{yang2025longlive} as our primary baselines. 
Training prompts are sampled from the VidProM dataset~\cite{wang2024vidprom}, specifically utilizing the filtered subset introduced in DanceGRPO~\cite{xue2025dancegrpo}. 
We employ Low-Rank Adaptation (LoRA) with rank $r=256$ and scaling factor $\alpha=256$ for parameter-efficient fine-tuning. 
To maximize memory efficiency during optimization, we do not store separate full-parameter copies for the current policy $v_\theta$ and the old policy $v_{\theta_{\text{old}}}$. 
Instead, both policies share a single frozen base model, and we switch between their respective lightweight LoRA during the forward pass, reducing GPU memory overhead.
Training operations are distributed across 48 NVIDIA H200 GPUs. 
Each epoch processes 48 prompts, maintaining a group size of $G=24$ candidate clips per prompt. 
For reward computation, we integrate VideoAlign~\cite{videoalign} and HPSv3~\cite{ma2025hpsv3} into our pipeline. 
More details can be found in supplementary material.

\begin{figure}[t]
    \centering
    \includegraphics[width=1.0\linewidth]{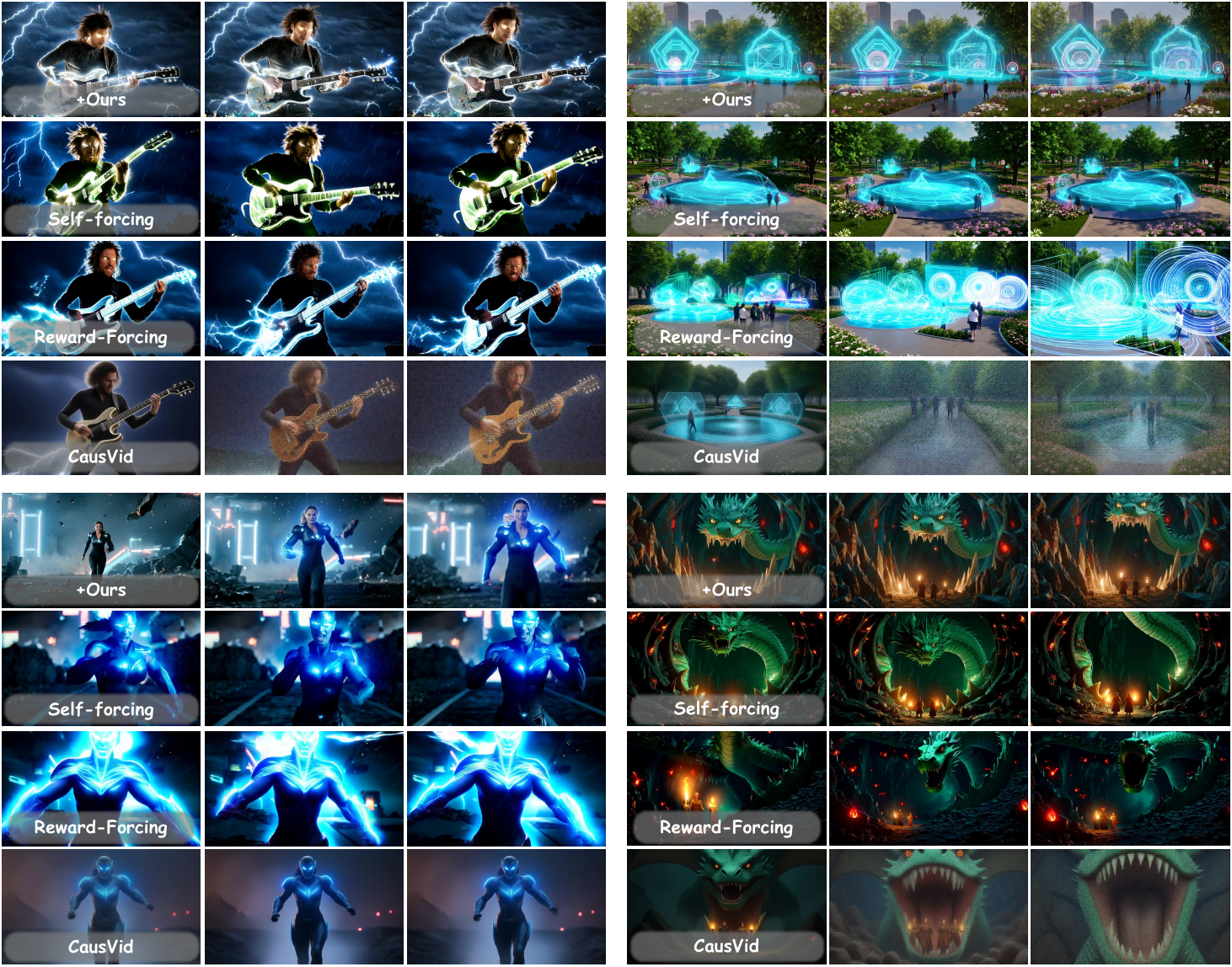}
    \vspace{-2mm}
    \caption{Qualitative comparison under the short-video, single-prompt setting. We evaluate our framework (+Ours) against other baselines. Visual results confirm that our method generates videos with significantly sharper textures and superior motion coherence, aligning better with human preferences. More results can be found in supplementary material.}
\vspace{-4mm}
\label{fig:compare_single}
\end{figure}

\begin{table}[t!]
\centering
\caption{Quantitative results on VBench benchmarks. Integrating our approach into existing distilled models yields consistent improvements in motion quality, semantic alignment, and overall generation quality.}
\label{tab:compare_short}
\resizebox{\textwidth}{!}{%
\begin{tabular}{lccccccc}
\toprule
\textbf{Method} & \textbf{Total}$\uparrow$ & \textbf{Quality}$\uparrow$ & \textbf{Semantic}$\uparrow$ & \textbf{HPSv3}$\uparrow$ & \textbf{MQ}$\uparrow$ & \textbf{Throughput}$\uparrow$ \\
\midrule
\rowcolor[gray]{0.95} \multicolumn{7}{l}{\textbf{Diffusion Models}} \\
LTX-Video~\cite{hacohen2024ltx} & 80.00 & 82.30 & 70.79 & 8.32 & 1.34 & 8.98 \\
Wan2.1~\cite{wan2025wan} & 84.26 & 85.30 & 80.09 & 9.26 & 1.62 & 0.78 \\
\midrule
\rowcolor[gray]{0.95} \multicolumn{7}{l}{\textbf{AR Models}} \\
SkyReels-V2~\cite{chen2025skyreels} & 82.67 & 84.70 & 74.53 & 9.08 & 1.59 & 0.49 \\
MAGI-1~\cite{teng2025magi} & 79.18 & 82.04 & 67.74 & 7.95 & 1.52 & 0.19 \\
NOVA~\cite{deng2024autoregressive} & 80.12 & 80.39 & 79.05 & 8.21 & 1.63 & 0.88 \\
PyramidFlow~\cite{jin2024pyramidal} & 81.72 & 84.74 & 69.62 & 8.76 & 1.50 & 6.70 \\
\midrule
\rowcolor[gray]{0.95} \multicolumn{7}{l}{\textbf{Distilled AR Models}} \\
CausVid~\cite{yin2025slow} & 81.20 & 84.05 & 69.80 & 7.56 & 1.22 & 17.0 \\
Reward Forcing~\cite{lu2025rewardforcing} & 84.13 & 84.84 & 81.32 & 8.74 & 1.65 & 23.1 \\
Self-Forcing~\cite{huang2025self} & 83.74 & 84.48 & 80.77 & 9.36 & 1.65 & 17.0 \\
\rowcolor{ourscolor} \quad + Ours & 83.79{\scriptsize\textcolor{green!60!black}{+.05}} & 84.51{\scriptsize\textcolor{green!60!black}{+.03}} & 80.92{\scriptsize\textcolor{green!60!black}{+.15}} & 10.72{\scriptsize\textcolor{green!60!black}{+1.36}} & 1.71{\scriptsize\textcolor{green!60!black}{+.06}} & 17.0 \\
LongLive~\cite{yang2025longlive} & 83.22 & 83.68 & 81.37 & 9.38 & 1.51 & 20.7 \\
\rowcolor{ourscolor} \quad + Ours & 84.93{\scriptsize\textcolor{green!60!black}{+1.71}} & 85.83{\scriptsize\textcolor{green!60!black}{+2.15}} & 81.36{\scriptsize\textcolor{red!60!black}{-.01}} & 11.03{\scriptsize\textcolor{green!60!black}{+1.65}} & 1.64{\scriptsize\textcolor{green!60!black}{+.13}} & 20.7 \\
Causal Forcing~\cite{zhu2026causal} & 84.04 & 84.59 & 81.84 & 9.48 & 1.69 & 17.0 \\
\rowcolor{ourscolor} \quad + Ours & 84.46{\scriptsize\textcolor{green!60!black}{+.42}} & 85.15{\scriptsize\textcolor{green!60!black}{+.56}} & 81.72{\scriptsize\textcolor{red!60!black}{-.12}} & 10.84{\scriptsize\textcolor{green!60!black}{+1.36}} & 1.80{\scriptsize\textcolor{green!60!black}{+.11}} & 17.0 \\
\bottomrule
\end{tabular}
}
\vspace{-4mm}
\end{table}

\begin{figure*}[t]
    \centering
    \includegraphics[width=1.0\linewidth]{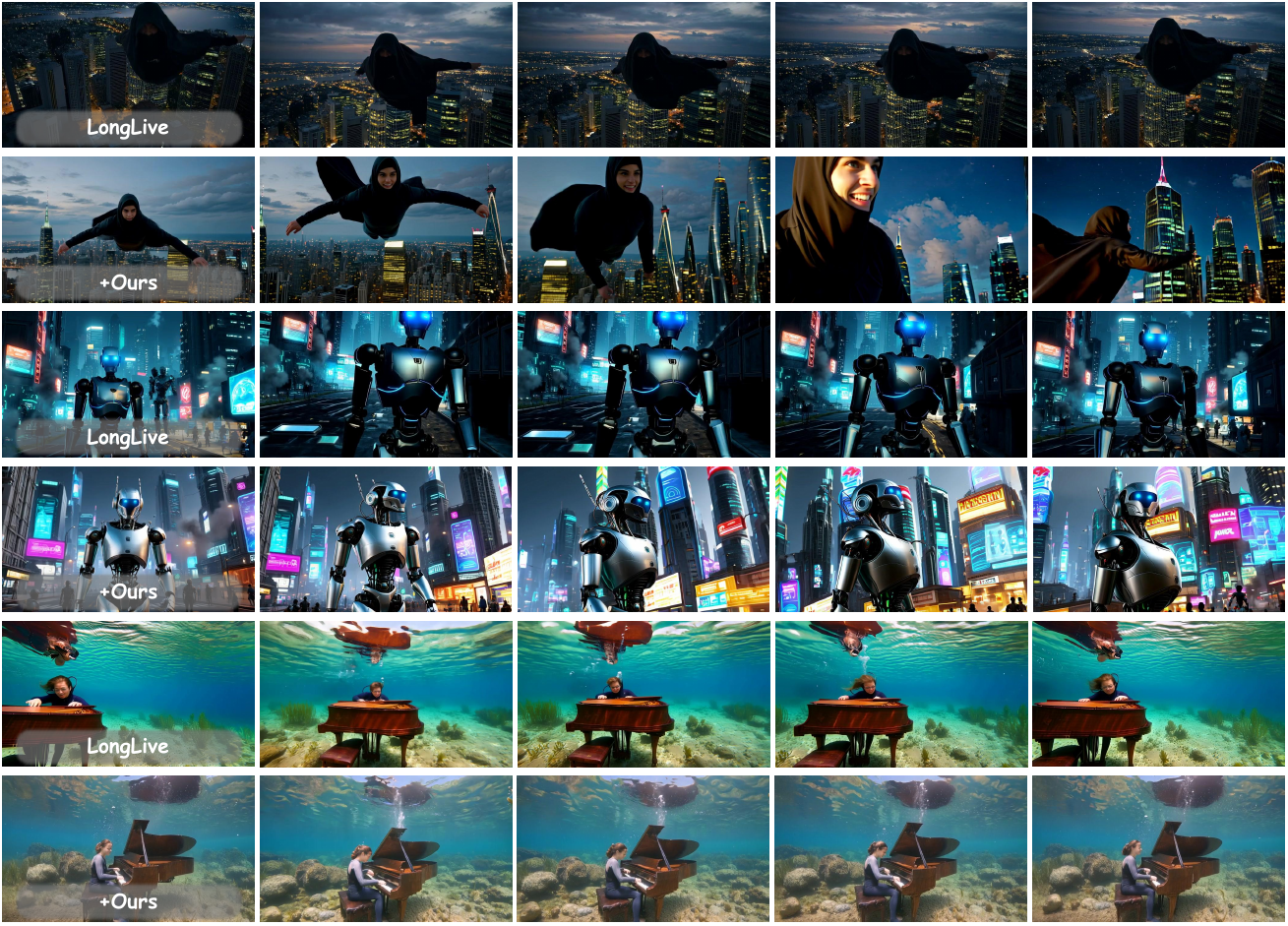}
    \vspace{-4mm}
    \caption{Qualitative results under the single-prompt long-video setting. Our framework (+Ours) effectively translates alignment optimizations from short videos to extended temporal horizons. Our approach delivers enhanced visual details and more stable throughout the sequence.}
\vspace{-4mm}
\label{fig:compare_single_long}
\end{figure*}

\begin{figure}[t]
  \centering
  
  \includegraphics[width=0.95\columnwidth]{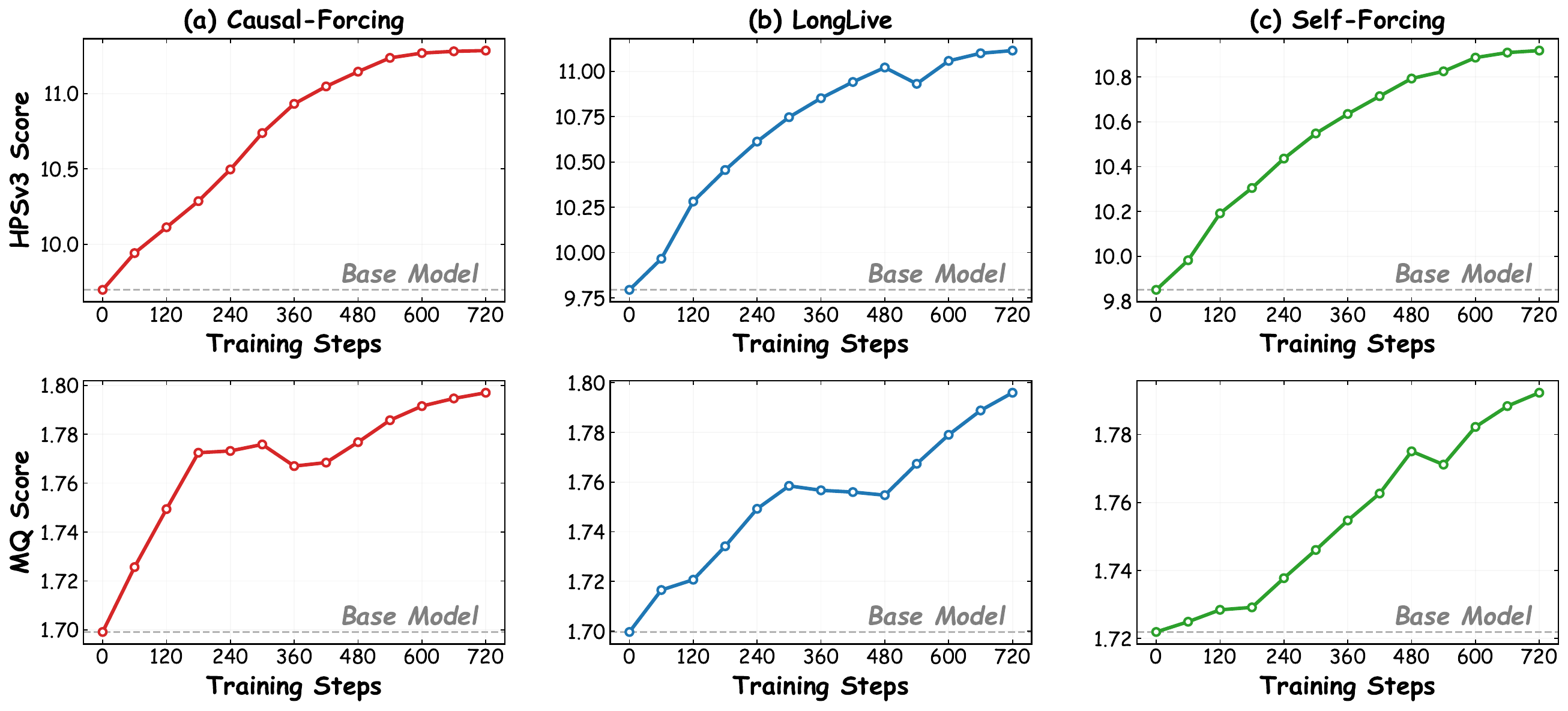}
  \vspace{-2mm}
  \caption{Performance improvements across different models. We evaluate our method on three models. The dashed grey lines indicate the baseline performance of the respective base models. The results demonstrate that our approach consistently improves both HPSv3 and MQ scores across all three models.}
  \label{fig:faces}
  \vspace{-2mm}
\end{figure}

\subsection{Short-Video Single-Prompt Generation}
\label{sec:short_5s}

We first validate our method under the short-video, single-prompt setting. 
Following VBench protocols~\cite{huang2024vbench}, we evaluate models using 946 standard prompts. 
To ensure a fair comparison with Self-Forcing, we utilize the augmented prompt test set during sampling, where prompts are expanded via Qwen2.5-7B-Instruct~\cite{bai2023qwen25} using Wan2.1~\cite{wan2025wan} system prompts. 
We integrate Astrolabe with various distilled AR models, comparing them against native AR models and bidirectional diffusion baselines. 
Quantitative results in Table~\ref{tab:compare_short} show that Astrolabe consistently enhances performance across all Self-Forcing variants. 
Similar gains observed in LongLive~\cite{yang2025longlive} and Causal-Forcing~\cite{zhu2026causal} further demonstrate the framework's generalizability across different base architectures.
To further assess alignment with human preferences, we curate 100 diverse prompts from MovieGenBench~\cite{polyak2024moviegen} for evaluation. 
We compute HPSv3 and Motion Quality scores to quantify improvements in aesthetic appeal and temporal consistency.
Results indicate that our RL-tuned models outperform their base versions in these metrics while maintaining the exact inference speed of the original checkpoints. 
Qualitative results in Figure~\ref{fig:compare_single} further confirm that Astrolabe yields sharper textures and superior motion coherence without sacrificing system throughput.

\begin{table}[t!]
  \centering
  \caption{Quantitative results on VBench-Long benchmarks. Integrating our method consistently improves the performance of long video generation baselines across both video quality and human preference metrics.}
  \label{tab:long_video}
  \begin{tabular}{lccccc}
    \toprule
    \textbf{Method} & \textbf{Total}$\uparrow$ & \textbf{Quality}$\uparrow$ & \textbf{Semantic}$\uparrow$ & \textbf{HPSv3}$\uparrow$ & \textbf{MQ}$\uparrow$ \\
    \midrule
    SkyReels-V2~\cite{chen2025skyreels}    & 75.29 & 80.77 & 53.37 & 8.72 & 1.54 \\
    FramePack~\cite{zhang2025frame}      & 81.95 & 83.61 & 75.32 & 8.94 & 1.58 \\
    \midrule
    Self-Forcing~\cite{huang2025self}   & 81.59 & 83.82 & 72.70 & 9.12 & 1.61 \\
    \rowcolor{ourscolor} \quad $+$ Ours & \textbf{82.03} & \textbf{84.36} & \textbf{72.71} & \textbf{10.38} & \textbf{1.72} \\
    LongLive~\cite{yang2025longlive}       & 83.52 & 85.44 & 75.82 & 9.21 & 1.48 \\
    \rowcolor{ourscolor} \quad $+$ Ours & \textbf{84.07} &\textbf{86.12} & \textbf{75.87} & \textbf{10.67} & \textbf{1.64} \\
    Causal Forcing~\cite{zhu2026causal} & 82.87 & 84.36 & \textbf{76.91} & 9.28 & 1.65 \\
    \rowcolor{ourscolor} \quad $+$ Ours & \textbf{84.24} & \textbf{86.18} & 76.48 & \textbf{10.52} & \textbf{1.74} \\
    \bottomrule
  \end{tabular}
  \vspace{-6mm}
\end{table}

\subsection{Long-Video Single-Prompt Generation}
\label{sec:long_30s}
%

\begin{figure*}[t!]
    \centering
    \includegraphics[width=1.0\linewidth]{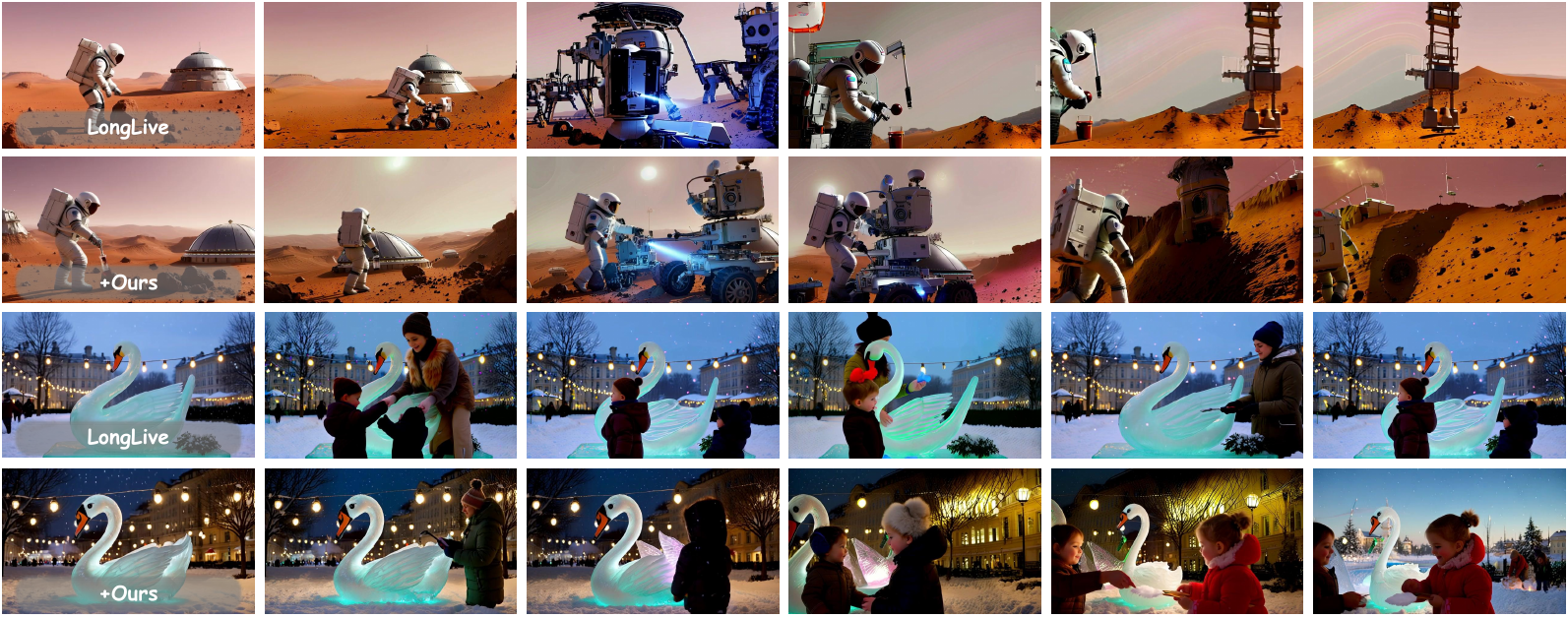}
    \vspace{-4mm}
    \caption{Qualitative comparison of multi-prompt long-video generation. We compare the LongLive~\cite{yang2025longlive} with our method. The generated sequences exhibit noticeable improvements in visual aesthetics and fine details during complex narrative transitions.}
\vspace{-4mm}
\label{fig:compare_multi}
\end{figure*}

Under the single-prompt long-video generation setting, we evaluate our method using VBench-Long protocols. 
For each prompt in the official dataset, we generate a 30-second video and subsequently partition it into localized clips using the standard VBench-Long evaluation scripts. 
Notably, while LongLive natively supports long-video generation, Self-Forcing and Causal-Forcing are exclusively trained on 5-second sequences. 
To enable long-horizon generation for these short-context models, we integrate the Infinity-RoPE~\cite{yesiltepe2025infinityrope} to extrapolate their positional embeddings. 
Furthermore, we rigorously benchmark these configurations against open-source solutions, including SkyReels-V2~\cite{chen2025skyreels} and FramePack~\cite{zhang2025frame}.
Quantitative results in Table~\ref{tab:long_video} report standard VBench-Long metrics measuring long-horizon quality and temporal consistency. 
Results indicate that our RL framework can also improve performance across long-video benchmarks, demonstrating that alignment optimizations conducted on short videos can effectively extrapolate to extended temporal horizons. 
Qualitative results in Figure~\ref{fig:compare_single_long} further confirm that Astrolabe yields sharper textures and superior motion coherence over extended durations.

\subsection{Long-Video Multi-Prompt Generation}
\label{sec:long_60s}

\begin{table}[t]
  \centering
  
  \caption{Quantitative evaluation on long video generation. We compare these overall metrics alongside CLIP Scores evaluated across 10-second intervals (0-60s).}
  \label{tab:long_video_clipscore}
  \setlength{\tabcolsep}{3.5pt}
  \resizebox{\linewidth}{!}{
  \begin{tabular}{l ccc cccccc}
    \toprule
    \multirow{2}{*}{\textbf{Method}}
      & \multirow{2}{*}{\textbf{Quality Score}$\uparrow$}
      & \multirow{2}{*}{\textbf{Consistency Score}$\uparrow$}
      & \multirow{2}{*}{\textbf{Aesthetic Score}$\uparrow$}
      
      & \multicolumn{6}{c}{\textbf{CLIP Score}$\uparrow$} \\
    \cmidrule(lr){5-10}
      &  &  &  & \textbf{0-10} & \textbf{10-20} & \textbf{20-30} & \textbf{30-40} & \textbf{40-50} & \textbf{50-60} \\
    \midrule
    SkyReels-V2~\cite{chen2025skyreels}    & 81.55 & 94.72 & 56.83 & 25.31 & 23.40 & 22.50 & 21.62 & 21.67 & 20.91 \\
    FramePack~\cite{zhang2025frame}      & 84.40 & 96.77 & 59.44 & 26.51 & 22.60 & 22.18 & 21.53 & 21.98 & 21.62 \\
    \midrule
    Self-Forcing~\cite{huang2025self} & 83.94 & 95.74 & 58.45 & 26.24 & \textbf{24.87} & 23.46 & 21.92 & \textbf{22.05} & 21.07 \\
    \rowcolor{ourscolor} + Ours & \textbf{84.72} & \textbf{95.98} & \textbf{59.62} & \textbf{26.42} & 24.75 & \textbf{23.95} & \textbf{22.40} & 21.85 & \textbf{21.50} \\
    \midrule
    LongLive~\cite{yang2025longlive}     & 84.28 & 96.05 & 59.89 & 26.63 & 25.77 & \textbf{24.65} & 23.99 & \textbf{24.52} & 24.11 \\
    \rowcolor{ourscolor} + Ours & \textbf{85.15} & \textbf{96.16} & \textbf{60.75} & \textbf{26.80} & \textbf{26.15} & 24.45 & \textbf{24.55} & 24.30 & \textbf{24.65} \\
    \midrule
    Causal-Forcing~\cite{zhu2026causal} & 84.12 & \textbf{95.88} & 59.15 & 26.45 & \textbf{25.60} & \textbf{23.98} & 22.85 & 22.48 & 22.45 \\
    \rowcolor{ourscolor} + Ours & \textbf{84.95} & 95.63 & \textbf{60.32} & \textbf{26.58} & 25.12 & 23.85 & \textbf{23.40} & \textbf{23.10} & \textbf{22.95} \\
    \bottomrule
  \end{tabular}
  }
  \vspace{-2mm}
\end{table}
To demonstrate that our framework effectively improves human preference alignment, we evaluate Astrolabe in the setting of interactive multi-prompt long-video generation.
We apply our method directly to the baselines, demonstrating how Astrolabe further enhances their capabilities.
Following established protocols from LongLive~\cite{yang2025longlive}, we curate 100 groups of narrative scripts. 
Each group comprises six successive 10-second prompts, yielding 60-second long-form videos. 
To ensure fair comparisons, short-context baselines (Self-Forcing, Causal-Forcing) are adapted for multi-prompt generation via prompt switching during the autoregressive rollout. 
LongLive, conversely, natively supports generative extrapolation with interactive instructions. 
We segment the generated videos at prompt boundaries to evaluate text alignment.
CLIP scores are subsequently computed at 10-second intervals to measure clip-wise semantic adherence.
Quantitative results in Table~\ref{tab:long_video_clipscore} show that Astrolabe improves overall generation quality, with noticeable gains in visual aesthetics and long-range motion consistency. 
Qualitative examples in Figure~\ref{fig:compare_multi} further illustrate these enhancements during extended video generation. 
These results suggest that our framework enhances both frame-level aesthetics and temporal consistency in complex multi-prompt setting.

\subsection{Ablation Studies}
\label{sec:ablation}

\begin{figure}[t!]
  \centering
  \begin{minipage}[b]{0.48\linewidth}
    \centering
    \includegraphics[width=\linewidth]{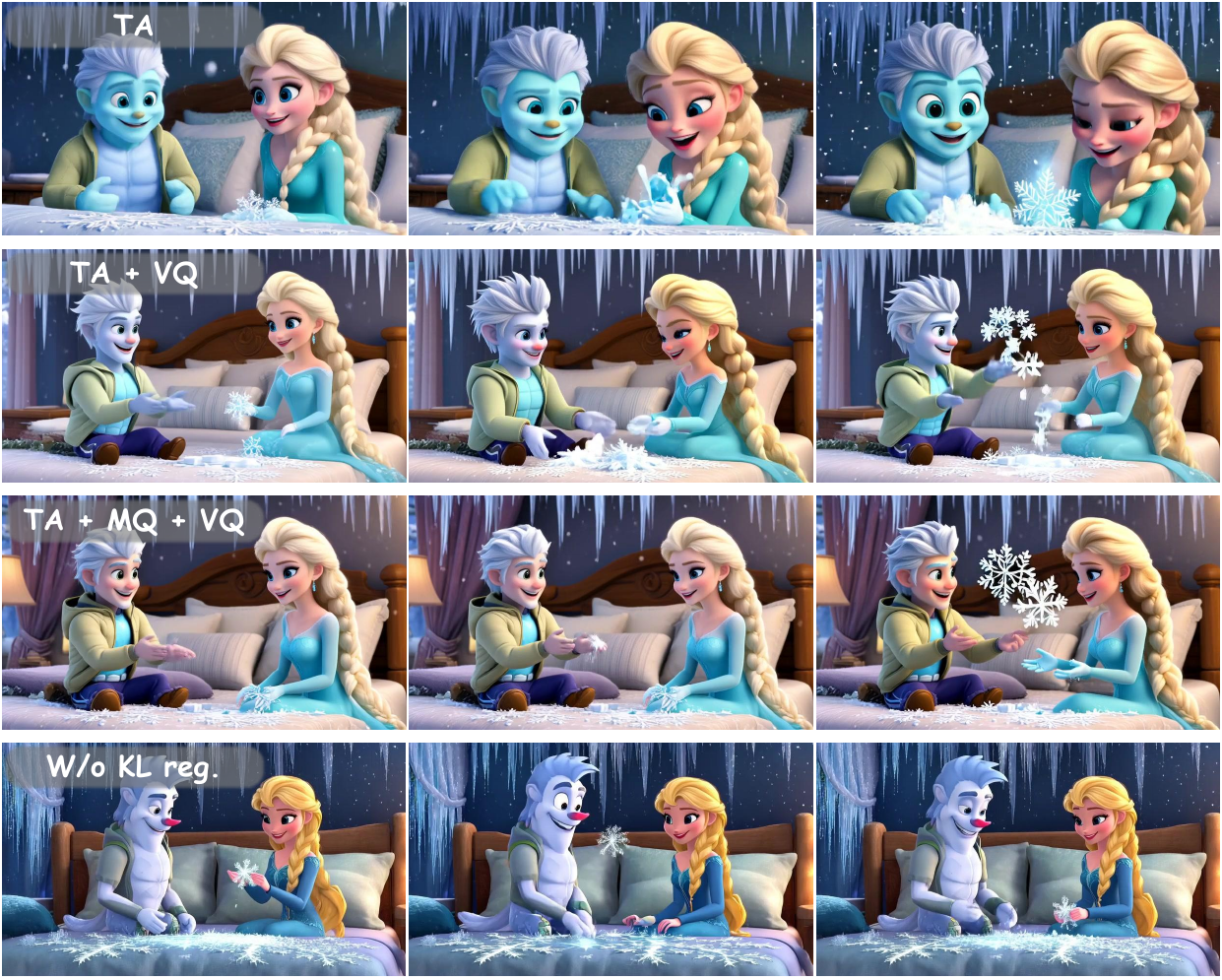}
    \vspace{1mm}
    \centerline{\small (a) Multi-reward design}
  \end{minipage}
  \hfill 
  \begin{minipage}[b]{0.48\linewidth}
    \centering
    \includegraphics[width=\linewidth]{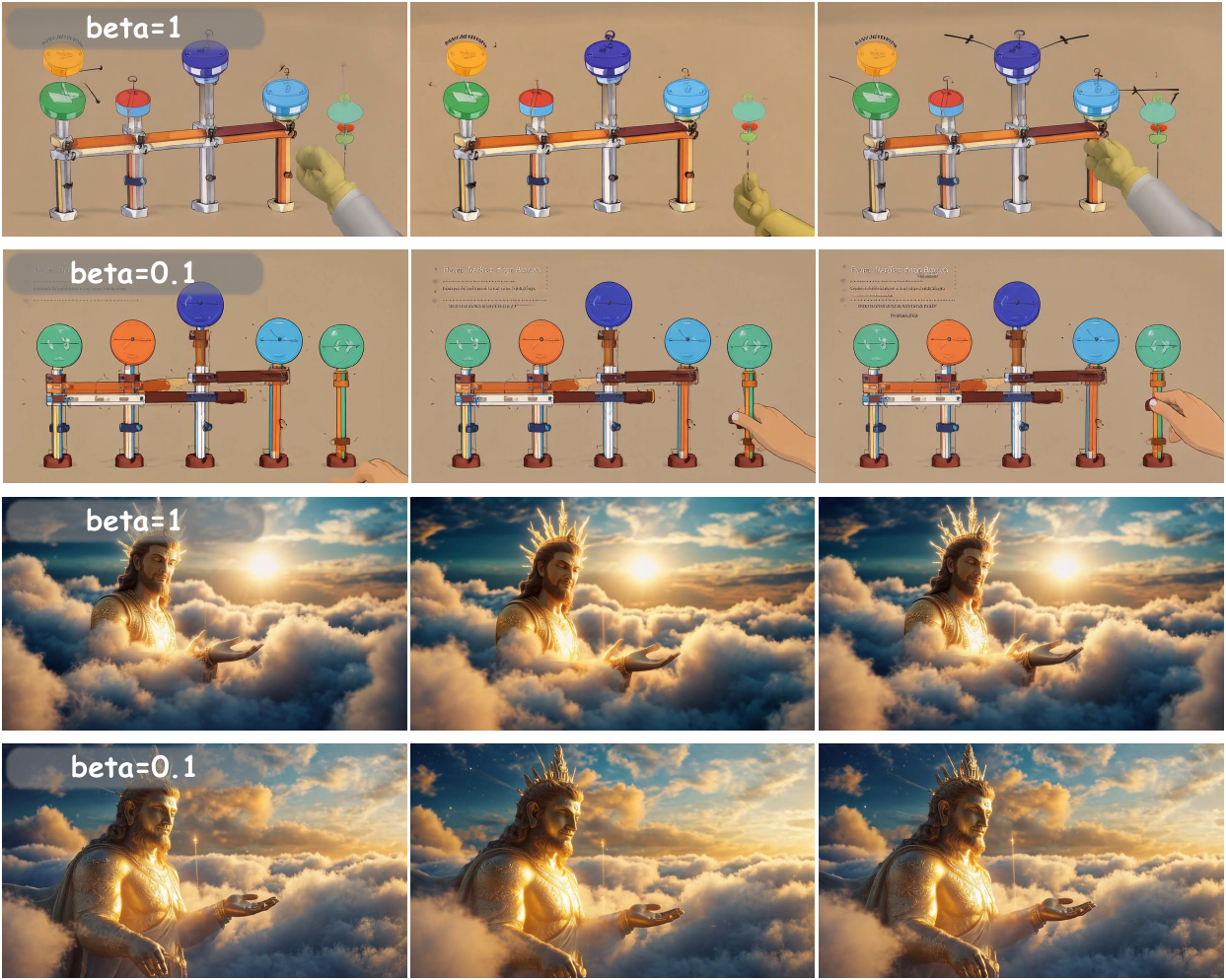}
    \vspace{1mm}
    \centerline{\small (b) Different $\beta$ values}
  \end{minipage}
  
  \vspace{-2mm}
  \caption{\textbf{Ablation studies on reward formulation and interpolation strength.} \textbf{(a)} Optimizing a single objective induces reward hacking and degrades other quality dimensions, whereas our aggregated formulation balances visual aesthetics and motion consistency. \textbf{(b)} The parameter $\beta$ controls the implicit contrast between positive and negative samples, with $\beta=1.0$ yielding the optimal trade-off for quality.}
  \vspace{-6mm}
  
  \label{fig:ablation_reward_beta}
\end{figure}

We conduct ablation studies to validate each component of our method.
All ablations are performed on Causal-Forcing with short-video alignment unless otherwise specified. 
Additional discussion and details can be found in the Supplementary Material.

\noindent{\textbf{Streaming Training Scheme.}}
Table~\ref{tab:ablation_main} compares different rollout and optimization strategies for 30-second video generation.
Sequence-level rollout with full backpropagation causes out-of-memory errors.
Our clip-level group-wise sampling with detached context achieves the best trade-off: it reduces memory consumption by $\approx2\times$ compared to clip-level full backpropagation while improving both HPSv3 and MQ.
The efficiency gains stem from sharing historical context across candidate clips reducing redundant computation.

\begin{table}[t]
\centering
\caption{Ablation studies on each component. \textbf{(a)} and \textbf{(b)} show the impact of the streaming strategy and selective KL regularization. \textbf{(c)} compares different reward combinations.}
\label{tab:ablation_main}
\vspace{-2mm}

\begin{minipage}[t]{\dimexpr(\textwidth-5mm)/2\relax}
  \centering
  \scriptsize
  \textbf{(a) Streaming Training}\vspace{1mm}\\
  \setlength{\tabcolsep}{3pt}
  \setlength{\aboverulesep}{0pt}
  \setlength{\belowrulesep}{0pt}
  \renewcommand{\arraystretch}{1.25}
  \begin{tabular}{lccc}
  \toprule
  \textbf{Config} & \textbf{HPSv3}$\uparrow$ & \textbf{MQ}$\uparrow$ & \textbf{Mem}$\downarrow$ \\
  \midrule\addlinespace[1pt]
  Seq + Full BP  & OOM & OOM & $>$140 \\
  Seq + Detach   & 10.21 & 1.72 & 96.4 \\
  Clip + Full BP & 10.58 & 1.76 & 112.3 \\
  \rowcolor{ourscolor} Clip+Detach & \textbf{10.84} & \textbf{1.80} & \textbf{54.3} \\
  \addlinespace[1pt]\bottomrule
  \end{tabular}
\end{minipage}%
\hspace{5mm}%
\begin{minipage}[t]{\dimexpr(\textwidth-5mm)/2\relax}
  \centering
  \scriptsize
  \textbf{(b) Selective KL Reg.}\vspace{1mm}\\
  \setlength{\tabcolsep}{3pt}
  \setlength{\aboverulesep}{0pt}
  \setlength{\belowrulesep}{0pt}
  \renewcommand{\arraystretch}{1.25}
  \begin{tabular}{lccc}
  \toprule
  \textbf{Strategy} & \textbf{HPSv3}$\uparrow$ & \textbf{MQ}$\uparrow$ & \textbf{TA}$\uparrow$ \\
  \midrule\addlinespace[1pt]
  No KL & 10.67 & 1.74 & -0.068 \\
  Uni. ($\lambda=1e^{-4}$) & 10.52 & 1.71 & 0.012 \\
  Uni. ($\lambda=5e^{-4}$) & 10.28 & 1.68 & 0.028 \\
  \rowcolor{ourscolor} Sel.+EMA & \textbf{10.84} & \textbf{1.80} & \textbf{0.065} \\
  \addlinespace[1pt]\bottomrule
  \end{tabular}
\end{minipage}

\vspace{4mm}

\begin{minipage}[t]{\textwidth}
  \centering
  \scriptsize
  \textbf{(c) Multi-Reward Formulation}\vspace{1mm}\\
  \setlength{\tabcolsep}{4pt}
  \setlength{\aboverulesep}{0pt}
  \setlength{\belowrulesep}{0pt}
  \renewcommand{\arraystretch}{1.25}
  \begin{tabular}{lcccc | lcccc}
  \toprule
  \textbf{Reward} & \textbf{HPSv3}$\uparrow$ & \textbf{MQ}$\uparrow$ & \textbf{TA}$\uparrow$ & \textbf{VB}$\uparrow$ &
  \textbf{Reward} & \textbf{HPSv3}$\uparrow$ & \textbf{MQ}$\uparrow$ & \textbf{TA}$\uparrow$ & \textbf{VB}$\uparrow$ \\
  \midrule\addlinespace[1pt]
  Baseline & 9.48  & 1.69 & -0.015 & 84.04 & VQ + MQ & 10.67 & 1.74 & -0.068 & 83.95 \\
  VQ only  & 10.92 & 1.58 & -0.075 & 83.21 & VQ + TA & 10.71 & 1.62 &  0.055 & 84.12 \\
  MQ only  & 9.31  & 1.82 & -0.058 & 83.67 & MQ + TA &  9.45 & 1.78 &  0.051 & 84.08 \\
  \rowcolor{ourscolor}
  \cellcolor{white} TA only &
  \cellcolor{white} 9.42  &
  \cellcolor{white} 1.62  &
  \cellcolor{white} 0.082 &
  \cellcolor{white} 84.25 &
  \textbf{All (Ours)} & \textbf{10.84} & \textbf{1.80} & \textbf{0.065} & \textbf{84.46} \\
  \addlinespace[1pt]\bottomrule
  \end{tabular}
\end{minipage}
\vspace{-2mm}
\end{table}

\noindent{\textbf{Reward Design and Regularization.}}
Table~\ref{tab:ablation_main} ablates our objective formulation. 
Single-reward optimization induces hacking: VQ-only training collapses into static frames, improving HPSv3 but degrading MQ. 
Our multi-reward formulation (VQ+MQ+TA) prevents this single-objective overfitting, yielding balanced improvements. Furthermore, uniform KL regularization over-constrains learning, while its omission causes instability and early MQ plateaus (Figure~\ref{fig:kl_penalty}). 
Our selective KL penalty with EMA reference updates resolves this by adaptively penalizing only high-uncertainty predictions. This targeted approach preserves optimization freedom for confident samples, effectively balancing exploration with stable convergence across all metrics.



\begin{figure*}[t] 
    \centering
    \begin{subfigure}[b]{0.49\linewidth}
        \centering
        \includegraphics[width=\linewidth]{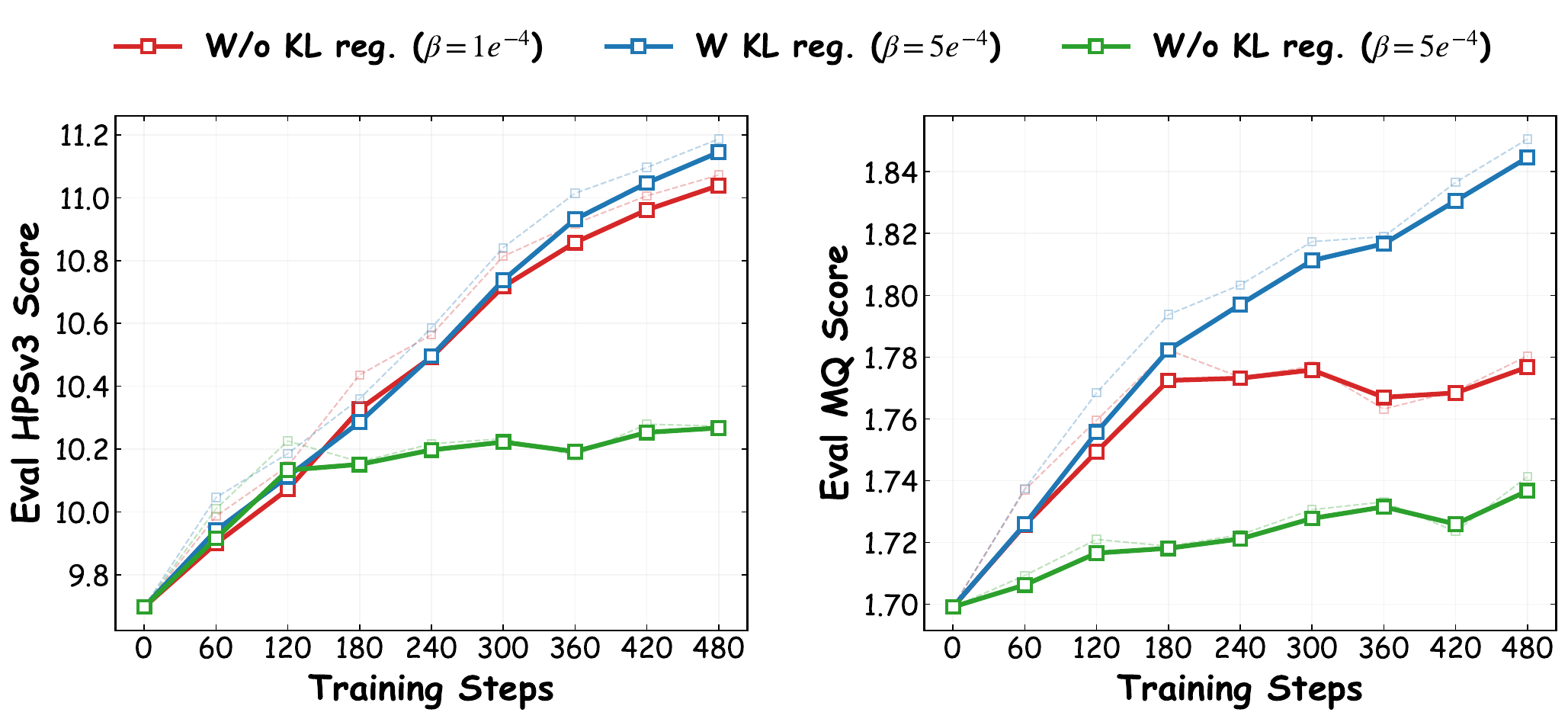}
        \caption{Ablation on KL penalty.}
        \label{fig:kl_penalty}
    \end{subfigure}
    \hfill
    \begin{subfigure}[b]{0.49\linewidth}
        \centering
        \includegraphics[width=\linewidth]{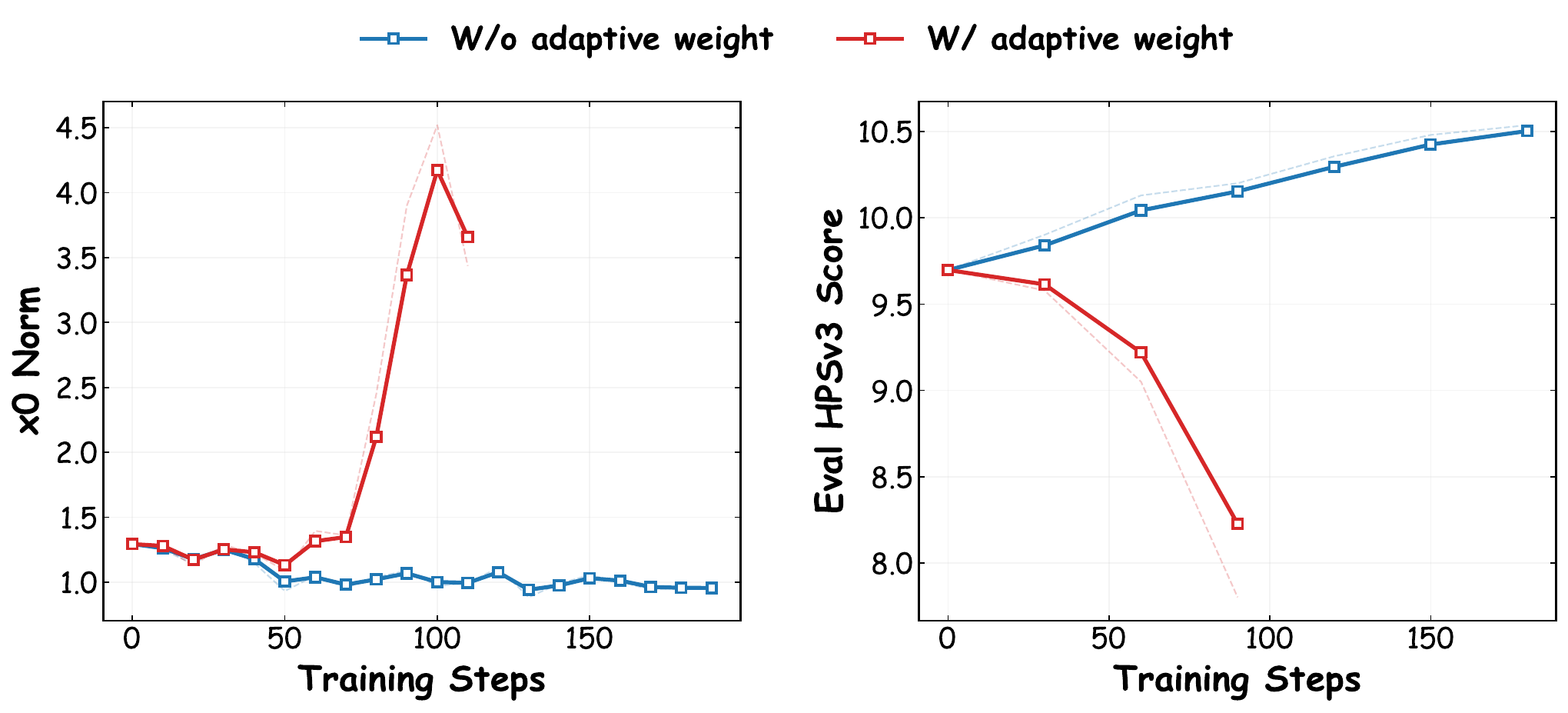}
        \caption{Ablation on adaptive weight.}
        \label{fig:adaptive_weight}
    \end{subfigure}
    
    \vspace{-1mm} 
    \caption{Ablation study of stabilization techniques in Astrolabe. (a) Effect of the selective KL penalty. (b) Impact of the dynamic adaptive weight mechanism.}
    \vspace{-4mm}
    \label{fig:adaptiveweight}
\end{figure*}

\noindent{\textbf{Removing Adaptive Weighting.}}
DiffusionNFT~\cite{zheng2025diffusionnft} scales the loss using a self-normalized $x_0$ denominator. 
However, Figure~\ref{fig:adaptive_weight} demonstrates this adaptive weighting destabilizes distilled AR setting. 
Under large discretization gaps, this dynamic denominator becomes volatile, causing the predicted $x_0$ norm to explode after 50 steps and triggering a sharp collapse in reward. Conversely, removing this scaling factor bounds the $x_0$ norm and ensures steady, monotonic reward improvements.

\noindent{\textbf{Impact of $\beta$.}}
Figure~\ref{fig:ablation_reward_beta} ablates the parameter $\beta$, which determines the scale of the implicit guidance direction integrated into the old policy. 
Empirical results indicate that varying $\beta$ directly influences the temporal dynamics of the generated sequences. In our experimental setup, setting $\beta=1$ yields higher overall visual and motion quality compared to a smaller value such as $\beta=0.1$. Consequently, we adopt $\beta=1$ as the default configuration to maintain generation stability.
\section{Conclusion}
\label{sec:conclusion}
We present Astrolabe, an online RL framework for aligning distilled autoregressive video models with human preferences. Utilizing a memory-efficient, forward-process RL formulation, our method eliminates the trajectory storage overhead of reverse-process alternatives. 
For long-video scalability, we introduce a streaming training scheme with local-window optimization, achieving constant peak memory. To prevent reward hacking, we implement a multi-reward formulation coupled with an uncertainty-aware selective KL penalty.
Extensive experiments across multiple distilled streaming architectures and benchmarks validate the effectiveness and generality of our approach.

\bibliographystyle{splncs04}
\bibliography{main}

\appendix

\setcounter{page}{1}
\setcounter{figure}{0}
\setcounter{table}{0}
\renewcommand{\thesection}{\Alph{section}}
\renewcommand{\thefigure}{S\arabic{figure}}
\renewcommand{\thetable}{S\arabic{table}}

\begin{algorithm}[t]
  \caption{NFT Training with Selective KL Regularization}
  \label{alg:diffusionnft}
  \begin{algorithmic}[1]
  \Require Policy $\pi_\theta$, behavior policy $\pi_{\theta_{\text{old}}}$, KL reference $\pi_{\theta_{\text{ref}}}$; reward functions $\{R_m\}_{m=1}^M$; prompts $\mathcal{D}$; hyperparameters $\beta$, $\lambda_{\text{KL}}$, $K_{\max}$, $\tau_{\text{KL}}$
  \Ensure Optimized policy $\pi_\theta$
  \State Initialize $\theta_{\text{old}} \leftarrow \theta$, $\theta_{\text{ref}} \leftarrow \theta$
  \State Initialize risk buffer $\mathcal{B} \leftarrow \emptyset$, KL ratio $\rho \leftarrow \rho_0$, $k_{\text{last}} \leftarrow 0$
  \For{each epoch $k$}
      \State \textbf{// Phase 1: Rollout and multi-reward evaluation}
      \State Sample prompts $\{c_i\}_{i=1}^B \sim \mathcal{D}$
      \State Generate samples $\{x_0^{(i)}\}$ using $\pi_{\theta_{\text{old}}}$
      \State Compute rewards $\{r_m^{(i)} = R_m(x_0^{(i)}, c_i)\}$ for each reward model $m$
      \State Compute advantages $\{A^{(i)}\}$ from primary reward
      \State \textbf{// Compute reward uncertainty via rank disagreement}
      \For{each reward model $m$}
          \State $\text{rank}_m^{(i)} \leftarrow \text{Rank}(\{r_m^{(j)}\}_{j=1}^B)$ \Comment{Rank samples by reward $m$}
      \EndFor
      \State $\Delta_{\text{rank}}^{(i)} \leftarrow \text{rank}_{\text{primary}}^{(i)} - \frac{1}{M-1}\sum_{m \neq \text{primary}} \text{rank}_m^{(i)}$
      \State \textbf{// Adaptive threshold via risk compensation}
      \State $\rho \leftarrow \textsc{UpdateRiskRatio}(\mathcal{B}, \{\Delta_{\text{rank}}^{(i)}\}, \rho)$
      \State $\tau \leftarrow \text{Percentile}(\{\Delta_{\text{rank}}^{(i)} | \Delta_{\text{rank}}^{(i)} \geq 0\}, 100 - 100\rho)$
      \State $\mathcal{M}^{(i)} \leftarrow \mathbbm{1}[\Delta_{\text{rank}}^{(i)} > \tau]$ \Comment{High-uncertainty mask}
      \State \textbf{// Phase 2: Forward process optimization}
      \For{each mini-batch $(x_0, A, \mathcal{M})$}
          \State Sample $t \sim \mathcal{U}(\mathcal{T})$
          \State $x_t \leftarrow (1-t)x_0 + t\epsilon$, \quad $\epsilon \sim \mathcal{N}(0, I)$
          \State $v_\theta, v_{\theta_{\text{old}}}, v_{\theta_{\text{ref}}} \leftarrow \pi_\theta(x_t, t), \pi_{\theta_{\text{old}}}(x_t, t), \pi_{\theta_{\text{ref}}}(x_t, t)$
          \State \textbf{// Policy loss}
          \State $v^+ \leftarrow \beta \cdot v_\theta + (1-\beta) \cdot v_{\theta_{\text{old}}}$
          \State $v^- \leftarrow (1+\beta) \cdot v_{\theta_{\text{old}}} - \beta \cdot v_\theta$
          \State $\tilde{r} \leftarrow \text{clip}(A / A_{\max}) / 2 + 0.5$
          \State $\mathcal{L}_{\text{policy}} \leftarrow \tilde{r} \|v^+ - x_0\|^2 + (1-\tilde{r}) \|v^- - x_0\|^2$
          \State \textbf{// Selective KL: only on high-uncertainty samples}
          \State $\mathcal{L}_{\text{KL}} \leftarrow \frac{1}{|\mathcal{M}|} \sum_{i: \mathcal{M}^{(i)}=1} \|v_\theta^{(i)} - v_{\theta_{\text{ref}}}^{(i)}\|^2$
          \State $\mathcal{L} \leftarrow \mathcal{L}_{\text{policy}} + \lambda_{\text{KL}} \cdot \mathcal{L}_{\text{KL}}$
          \State Update: $\theta \leftarrow \theta - \eta \nabla_\theta \mathcal{L}$
      \EndFor
      \State \textbf{// Adaptive reference update based on KL divergence}
      \If{$\mathcal{L}_{\text{KL}} > \tau_{\text{KL}}$ \textbf{or} $k - k_{\text{last}} > K_{\max}$}
          \State $\theta_{\text{ref}} \leftarrow \theta$ \Comment{Reset KL reference}
          \State $k_{\text{last}} \leftarrow k$
      \EndIf
      \State $\theta_{\text{old}} \leftarrow \gamma \cdot \theta_{\text{old}} + (1-\gamma) \cdot \theta$
  \EndFor
  \end{algorithmic}
\end{algorithm}

\begin{algorithm}[t]
  \caption{Streaming Training for Long Video}
  \label{alg:streaming_diffusionnft}
  \begin{algorithmic}[1]
  \Require Policy $\pi_\theta$, behavior policy $\pi_{\theta_{\text{old}}}$, KL reference $\pi_{\theta_{\text{ref}}}$; rewards $\{R_m\}$; prompts $\mathcal{D}$; window size $W$,
  total frames $F$
  \Ensure Optimized policy $\pi_\theta$
  \State Initialize $\theta_{\text{old}} \leftarrow \theta$, $\theta_{\text{ref}} \leftarrow \theta$, risk buffer $\mathcal{B} \leftarrow \emptyset$, $\rho \leftarrow \rho_0$
  \For{each epoch $k$}
      \State \textbf{// Select training window (synced across GPUs)}
      \State $s \leftarrow \textsc{SelectWindow}(F, W)$, \quad $F_{\text{req}} \leftarrow s + W$
      \State \textbf{// Phase 1: Efficient rollout with multi-reward}
      \State Sample prompts $\{c_i\}_{i=1}^B \sim \mathcal{D}$
      \State Generate partial videos $\{x_0^{(i)}[0:F_{\text{req}}]\}$ using $\pi_{\theta_{\text{old}}}$
      \State Extract window: $x_{0,W}^{(i)} \leftarrow x_0^{(i)}[s:s+W]$
      \State Compute rewards $\{r_m^{(i)} = R_m(x_{0,W}^{(i)}, c_i)\}$ for each model $m$
      \State Compute advantages $\{A^{(i)}\}$
      \State \textbf{// Reward uncertainty via rank disagreement}
      \For{each reward model $m$}
          \State $\text{rank}_m^{(i)} \leftarrow \text{Rank}(\{r_m^{(j)}\}_{j=1}^B)$
      \EndFor
      \State $\Delta_{\text{rank}}^{(i)} \leftarrow \text{rank}_{\text{primary}}^{(i)} - \frac{1}{M-1}\sum_{m \neq \text{primary}} \text{rank}_m^{(i)}$
      \State $\rho \leftarrow \textsc{UpdateRiskRatio}(\mathcal{B}, \{\Delta_{\text{rank}}^{(i)}\}, \rho)$
      \State $\tau \leftarrow \text{Percentile}(\{\Delta_{\text{rank}}^{(i)} \geq 0\}, 100 - 100\rho)$
      \State $\mathcal{M}^{(i)} \leftarrow \mathbbm{1}[\Delta_{\text{rank}}^{(i)} > \tau]$
      \State \textbf{// Phase 2: Window-focused optimization with selective KL}
      \For{each mini-batch $(x_0, A, \mathcal{M})$}
          \State Sample $t \sim \mathcal{U}(\mathcal{T})$, \quad $x_t \leftarrow (1-t)x_0 + t\epsilon$
          \State $v_\theta, v_{\theta_{\text{old}}}, v_{\theta_{\text{ref}}} \leftarrow \textsc{StreamingNFT}(x_t, t, s, W)$
          \State Extract: $v_{\theta,W}, v_{\theta_{\text{old}},W}, v_{\theta_{\text{ref}},W}, x_{0,W}$
          \State \textbf{// Policy loss on window}
          \State $v^+ \leftarrow \beta \cdot v_{\theta,W} + (1-\beta) \cdot v_{\theta_{\text{old}},W}$
          \State $v^- \leftarrow (1+\beta) \cdot v_{\theta_{\text{old}},W} - \beta \cdot v_{\theta,W}$
          \State $\tilde{r} \leftarrow \text{clip}(A / A_{\max}) / 2 + 0.5$
          \State $\mathcal{L}_{\text{policy}} \leftarrow \tilde{r} \|v^+ - x_{0,W}\|^2 + (1-\tilde{r}) \|v^- - x_{0,W}\|^2$
          \State \textbf{// Selective KL on high-uncertainty samples}
          \State $\mathcal{L}_{\text{KL}} \leftarrow \frac{1}{|\mathcal{M}|} \sum_{i: \mathcal{M}^{(i)}=1} \|v_{\theta,W}^{(i)} - v_{\theta_{\text{ref}},W}^{(i)}\|^2$
          \State Update: $\theta \leftarrow \theta - \eta \nabla_\theta (\mathcal{L}_{\text{policy}} + \lambda_{\text{KL}} \mathcal{L}_{\text{KL}})$
      \EndFor
      \State \textbf{// Adaptive reference update}
      \If{$\mathcal{L}_{\text{KL}} > \tau_{\text{KL}}$ \textbf{or} $k - k_{\text{last}} > K_{\max}$}
          \State $\theta_{\text{ref}} \leftarrow \theta$, \quad $k_{\text{last}} \leftarrow k$
      \EndIf
      \State $\theta_{\text{old}} \leftarrow \gamma \cdot \theta_{\text{old}} + (1-\gamma) \cdot \theta$
  \EndFor
  \end{algorithmic}
  \end{algorithm}

\section{Hyperparameters}
\label{sec:supp_hyper}
We provide detailed hyperparameter in Table~\ref{tab:hyperparameters}.

\begin{table}[h]
\centering
\caption{Comprehensive Hyperparameters for Astrolabe Training. We detail the configurations used across the model architecture, optimization, diffusion process, reinforcement learning, and streaming rollout.}
\label{tab:hyperparameters}
\resizebox{0.95\textwidth}{!}{
\begin{tabular}{llc}
\toprule
\textbf{Module} & \textbf{Hyperparameter} & \textbf{Value} \\
\midrule
\multirow{2}{*}{\textbf{Model \& Video Specs}} 
& Base Architecture & Causal Wan 2.1 \\
& Video Resolution ($H \times W$) & $480 \times 832$ \\
\midrule
\multirow{4}{*}{\textbf{LoRA Fine-Tuning}} 
& Rank ($r$) & 256 \\
& Scaling Factor ($\alpha$) & 256 \\
& Dropout Rate & 0.0 \\
& Gradient Checkpointing & True \\
\midrule
\multirow{7}{*}{\textbf{Optimization}} 
& Hardware & 48 $\times$ NVIDIA GPUs \\
& Precision Mode & \texttt{bf16} (Mixed Precision) \\
& Optimizer & AdamW ($\beta_1=0.9, \beta_2=0.999$) \\
& Learning Rate ($\eta$) & 1e-5 \\
& Weight Decay & 1e-4 \\
& Epsilon ($\epsilon$) & 1e-8 \\
& Max Gradient Norm & 1.0 \\
\midrule
\multirow{3}{*}{\textbf{Diffusion Process}} 
& Distillation Timesteps ($T$) & 4 (sampled from $[1000, 750, 500, 250]$) \\
& Timestep Shift & 5.0 \\
& Forward Process Noise Level & 0.7 \\
\midrule
\multirow{6}{*}{\textbf{Selective Regularization}} 
& Interpolation Strength ($\beta$) & 1.0 \\
& KL Penalty Weight ($\lambda_{KL}$) & 1e-4 \\
& Advantage Clip Max & 5.0 \\
& Reward Normalization & Global Std with Per-Prompt Tracking \\
& EMA Decay Rate ($\gamma$) & 0.9 \\
& EMA Update Interval & 1 step \\
\midrule
\multirow{3}{*}{\textbf{Streaming Rollout}} 
& Window Selection Strategy & Random Choice \\
& Rolling Window Size ($L$) & 21 (Self-Forcing) / 15 (Longlive) \\
& Frame Sink Size ($S$) & 3 \\
\bottomrule
\end{tabular}
}
\end{table}




\section{Theoretical Proofs and Analysis}
\label{sec:appendix_theory}

We extend the theoretical foundation of DiffusionNFT~\cite{zheng2025diffusionnft} from global, non-autoregressive continuous-time diffusion to the settings of autoregressive (AR) streaming generation and few-step distilled models. Specifically, we prove the optimality of local advantage guidance (Theorem~\ref{thm:advantage_guidance}), and establish a reward lower bound via selective KL penalty (Theorem~\ref{thm:kl_lower_bound}).

\subsection{Conditional Improvement via Advantage Guidance}

\begin{theorem}[Conditional Improvement via Advantage Guidance]
\label{thm:advantage_guidance}
Let $\mathcal{C}_n$ be the frozen context at step $n$, $\tilde{r}(x_n, \mathcal{C}_n) \in [0,1]$ the normalized advantage, and $\alpha \triangleq \alpha(x_n^t, \mathcal{C}_n) = \mathbb{E}_{\pi_{\mathrm{old}}}[\tilde{r} \mid x_n^t, \mathcal{C}_n]$ the posterior positive probability. Define the implicit positive/negative distributions:
\begin{equation}
    \pi^+(x_n \mid \mathcal{C}_n) = \frac{\tilde{r} \cdot \pi_{\mathrm{old}}}{\mathbb{E}[\tilde{r}]}, \quad
    \pi^-(x_n \mid \mathcal{C}_n) = \frac{(1-\tilde{r}) \cdot \pi_{\mathrm{old}}}{\mathbb{E}[1-\tilde{r}]}
\end{equation}
and the local policy loss ($\beta > 0$ controls negative repulsion strength):
\begin{equation}
    \mathcal{L}_{\mathrm{policy}}^{(n)}(\theta) = \mathbb{E}_{t,\, x_n^t} \Big[ \| v_\theta - v^+ \|^2 + \beta \| v_\theta - v^- \|^2 \Big]
\end{equation}
where $v^{\pm}$ are conditional flow matching targets under $\pi^{\pm}$. Then the optimal velocity field satisfies:
\begin{equation}
    v_\theta^* = v_{\mathrm{old}} + \frac{1 - \alpha(1+\beta)}{(1+\beta)(1-\alpha)} \big( v^+ - v_{\mathrm{old}} \big)
\end{equation}
which strictly shifts the local transition toward the advantage-weighted policy in the typical regime $\alpha(1+\beta) < 1$.
\end{theorem}

\begin{proof}
Since the joint distribution factorizes as $p(x_{1:N}) = \prod_n p(x_n \mid \mathcal{C}_n)$, the velocity field at step $n$ is strictly conditioned on $\mathcal{C}_n$. By Bayes' theorem, the posterior decomposes as:
\begin{equation}
    \pi_{\mathrm{old}}(x_n^0 \mid x_n^t, \mathcal{C}_n) = \alpha \, \pi^+(x_n^0 \mid x_n^t, \mathcal{C}_n) + (1-\alpha) \, \pi^-(x_n^0 \mid x_n^t, \mathcal{C}_n)
\end{equation}
By the law of total expectation, the baseline velocity inherits this mixture:
\begin{equation} \label{eq:v_old_mix}
    v_{\mathrm{old}} = \alpha \, v^+ + (1-\alpha) \, v^-
\end{equation}
Let $\delta^+ = v^+ - v_{\mathrm{old}}$. From Eq.~\eqref{eq:v_old_mix}, $v^- - v_{\mathrm{old}} = -\frac{\alpha}{1-\alpha}\delta^+$. Substituting into the loss and writing $u = v_\theta - v_{\mathrm{old}}$:
\begin{equation}
    \mathcal{L} = \mathbb{E}\Big[ \|u - \delta^+\|^2 + \beta \big\|u + \tfrac{\alpha}{1-\alpha}\delta^+\big\|^2 \Big]
\end{equation}
Setting $\partial \mathcal{L}/\partial u = 0$:
\begin{equation}
    (1+\beta)\, u^* = \delta^+ - \frac{\beta\alpha}{1-\alpha}\,\delta^+ = \frac{1-\alpha(1+\beta)}{1-\alpha}\,\delta^+
\end{equation}
which yields the stated result. When $\alpha > 0$ and $\alpha(1+\beta) < 1$, the shift aligns with $v^+ - v_{\mathrm{old}}$, pushing the model toward the high-advantage region. As $\beta \to 0$, this degrades to standard reward-weighted regression. Thus, optimizing the local advantage strictly improves $\pi_\theta(x_n \mid \mathcal{C}_n)$ without requiring gradients across the full trajectory.
\end{proof}

\subsection{Performance Lower Bound with Selective Trust Region}

\begin{theorem}[Performance Lower Bound with Selective Trust Region]
\label{thm:kl_lower_bound}
Let $\hat{R}(x)$ be the proxy reward, $R^*(x)$ the true preference with $|R^*| \leq R_{\max}$, and $\mathcal{U} = \{x \mid \Delta_{\mathrm{rank}}(x) > \tau\}$ the high-uncertainty region. Assume:
\begin{itemize}
    \item \textbf{(A3)} For $x \notin \mathcal{U}$: $|\hat{R}(x) - R^*(x)| \leq \epsilon_{\mathrm{safe}}$.
    \item \textbf{(A4)} Under $\pi_{\mathrm{ref}}$: $\mathbb{E}_{\pi_{\mathrm{ref}}}[|\hat{R} - R^*| \mid x \in \mathcal{U}] \leq \epsilon_{\mathrm{risk}}$, with $\epsilon_{\mathrm{risk}} \gg \epsilon_{\mathrm{safe}}$.
\end{itemize}
Define the selective velocity MSE in $\mathcal{U}$:
\begin{equation}
    \mathcal{L}_{\mathrm{KL}}(\theta) = \mathbb{E}_{x \sim \pi_\theta,\, x \in \mathcal{U}} \int_0^1 \frac{\|v_\theta(x_t,t) - v_{\mathrm{ref}}(x_t,t)\|^2}{\sigma^2(t)} \, dt
\end{equation}
Then the true reward of $\pi_\theta$ is lower-bounded by:
\begin{equation}
    \mathbb{E}_{\pi_\theta}[R^*] \geq \mathbb{E}_{\pi_\theta}[\hat{R}] - \epsilon_{\mathrm{safe}} - \pi_\theta(\mathcal{U}) \big( \epsilon_{\mathrm{risk}} + 2R_{\max}\sqrt{\mathcal{L}_{\mathrm{KL}}} \big)
\end{equation}
\end{theorem}

\begin{proof}
\textbf{Step 1: Regional decomposition.}
\begin{equation}
    \mathbb{E}_{\pi_\theta}[R^*] = \mathbb{E}_{\pi_\theta}[R^* \mid x \notin \mathcal{U}] \cdot \pi_\theta(\mathcal{U}^C) + \mathbb{E}_{\pi_\theta}[R^* \mid x \in \mathcal{U}] \cdot \pi_\theta(\mathcal{U})
\end{equation}

\textbf{Step 2: Safe region.} By (A3), $\mathbb{E}_{\pi_\theta}[R^* \mid x \notin \mathcal{U}] \geq \mathbb{E}_{\pi_\theta}[\hat{R} \mid x \notin \mathcal{U}] - \epsilon_{\mathrm{safe}}$.

\textbf{Step 3: Risk region.} Since (A4) holds for $\pi_{\mathrm{ref}}$, transferring to $\pi_\theta$ requires bounding the distribution shift. By the TV dual representation with $|f| \leq 2R_{\max}$:
\begin{equation}
    \mathbb{E}_{\pi_\theta}[R^* \mid \mathcal{U}] \geq \mathbb{E}_{\pi_\theta}[\hat{R} \mid \mathcal{U}] - \epsilon_{\mathrm{risk}} - 4R_{\max} \cdot D_{\mathrm{TV}}(\pi_\theta \| \pi_{\mathrm{ref}} \mid \mathcal{U})
\end{equation}

\textbf{Step 4: KL--TV connection.} By Girsanov's theorem, the KL divergence restricted to $\mathcal{U}$ equals $\frac{1}{2}\mathcal{L}_{\mathrm{KL}}$. Pinsker's inequality then gives:
\begin{equation}
    D_{\mathrm{TV}} \leq \sqrt{\tfrac{1}{2} D_{\mathrm{KL}}} = \tfrac{1}{2}\sqrt{\mathcal{L}_{\mathrm{KL}}}
\end{equation}

\textbf{Step 5: Assembly.} Combining Steps 1--4 with $\pi_\theta(\mathcal{U}^C) \leq 1$ yields the stated bound.
\end{proof}

\textbf{Theoretical Insight.} Unlike global KL penalties that indiscriminately constrain all exploration, Astrolabe's selective penalty acts only through the $\pi_\theta(\mathcal{U}) \cdot 2R_{\max}\sqrt{\mathcal{L}_{\mathrm{KL}}}$ term. In safe regions ($x \notin \mathcal{U}$), the policy explores freely; the trust region activates only when $\pi_\theta$ drifts into the high-uncertainty region $\mathcal{U}$, naturally balancing exploration and stability.

\section{Algorithm}
We provide the detailed pseudocode for the training procedures of the Astrolabe framework. Algorithm~\ref{alg:diffusionnft} outlines the core forward-process reinforcement learning pipeline, termed NFT Training with Selective KL Regularization. 
It details the multi-reward evaluation, uncertainty quantification via rank disagreement, and the selective application of the KL penalty to mitigate reward hacking during short-clip optimization. 
To scale this alignment to long-sequence generation without exceeding memory constraints, Algorithm~\ref{alg:streaming_diffusionnft} presents the Streaming Training for Long Video. 
This algorithm leverages a rolling KV cache and detaches historical context, applying gradient updates exclusively to local active windows while maintaining long-range temporal coherence.

\section{More Qualitative Results}
\label{sec:supp_qualitative}
We present additional qualitative comparisons on 30-second long video generation under the single-prompt setting in Figures~\ref{fig:comparison_00007_00011}--\ref{fig:comparison_00959_00971}.
Qualitative comparisons showing improvements over baseline distilled models and competing methods.
As shown in Figure~\ref{fig:krea}, our method consistently produces higher-quality and more temporally coherent videos compared to the Krea 14B baseline~\cite{krea_realtime_14b}, demonstrating that our approach yields notable improvements even when applied on top of large-scale 14B models.

\begin{figure}[h]
    \centering
    \includegraphics[width=1.0\linewidth]{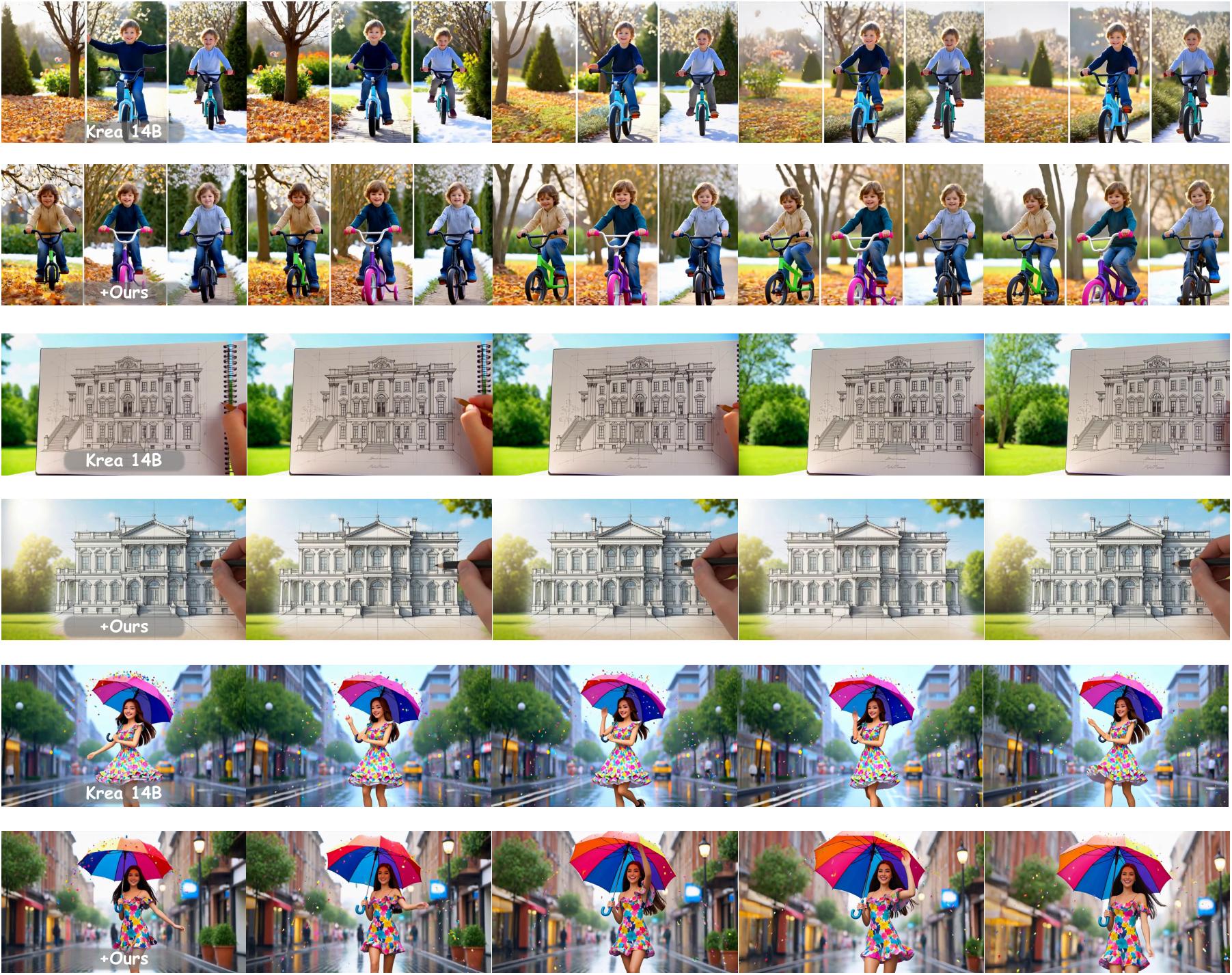}
    \vspace{-2mm}
    \caption{Qualitative comparison between Krea 14B (odd rows) and with our method (+Ours, even rows) }
\vspace{-2mm}
\label{fig:krea}
\end{figure}

\section{Discussion}
While Astrolabe establishes a robust and memory-efficient paradigm for aligning distilled autoregressive video models, several limitations warrant further investigation.

\noindent{\textbf{Reliance on the Capability of Reward Models.}}
Astrolabe fundamentally depends on the accuracy of the selected reward models.
Currently, open-source video evaluation models (e.g., VideoAlign) exhibit strong proficiency in assessing short-term aesthetics and text-video alignment. 
However, their ability to evaluate complex physics, long-horizon causality, and multi-entity interactions in minute-scale videos remains limited. Consequently, if the reward model fails to penalize subtle temporal hallucinations in extended sequences, Astrolabe cannot explicitly correct them. Developing more robust, physics-aware reward models for long-form video represents a critical next step for the community.

\noindent{\textbf{Inherent Bottlenecks of the Base Architecture.}}
As a post-training RL framework, Astrolabe excels at shifting the generation distribution toward high-reward regions and mitigating error accumulation. However, reinforcement learning cannot arbitrarily instantiate capabilities that are entirely absent from the distilled base model. For instance, if the original streaming models (e.g., Self-Forcing or LongLive) inherently lack the capacity to render highly complex spatial geometries or specific domain knowledge due to extreme distillation, our method can optimize the presentation of existing knowledge but cannot overcome the fundamental capacity ceiling of the base architecture.


\begin{figure*}[t]
    \centering
    \includegraphics[width=\linewidth]{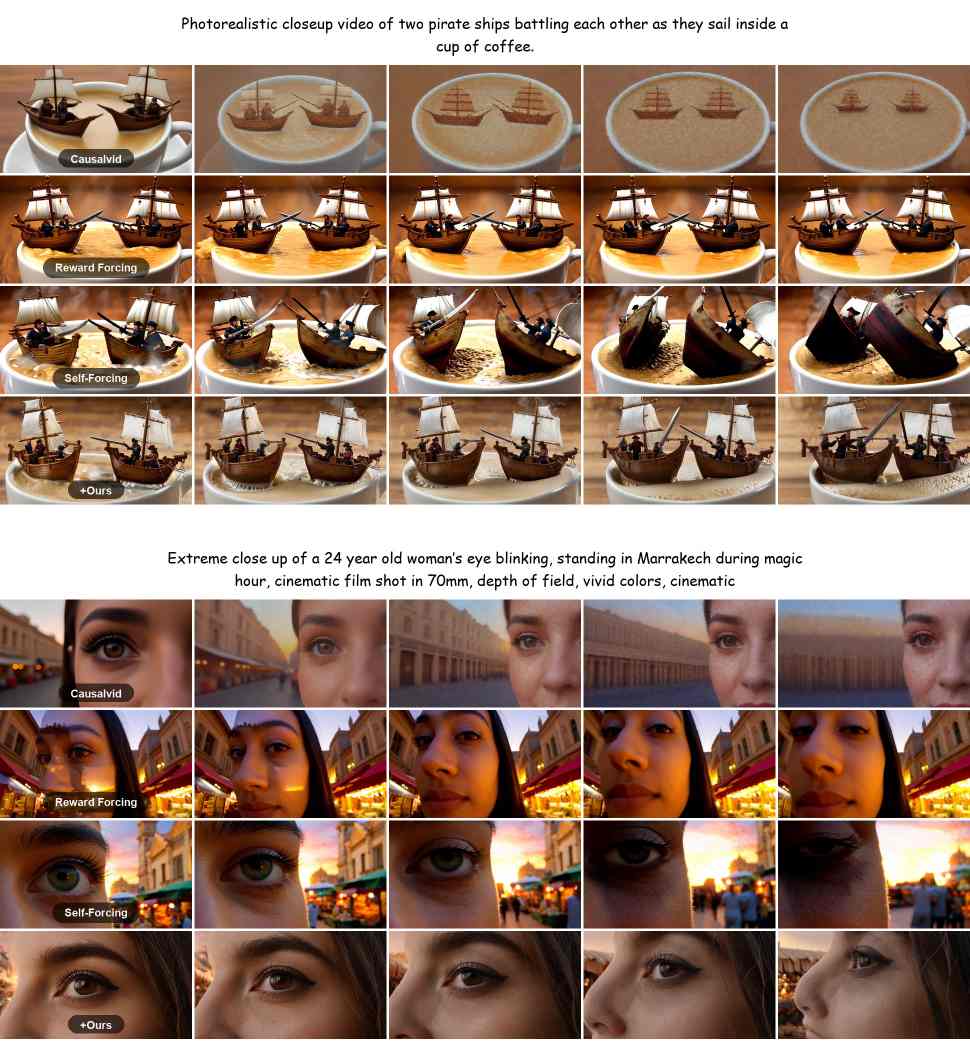}
    \caption{Qualitative comparison on long video generation (30s) under the single-prompt setting (Set 1).}
    \label{fig:comparison_00007_00011}
\end{figure*}

\begin{figure*}[t]
    \centering
    \includegraphics[width=\linewidth]{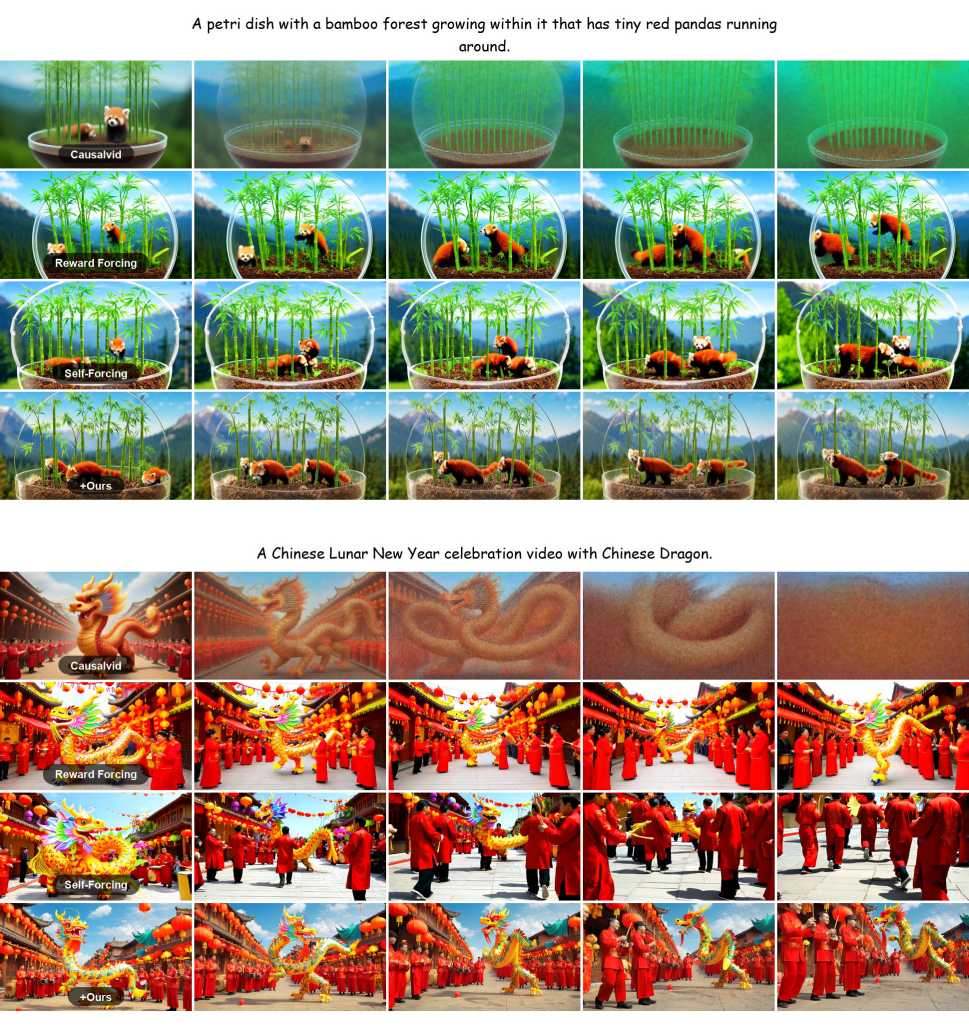}
    \caption{Additional qualitative comparison on long video generation (30s) under the single-prompt setting (Set 2).}
    \label{fig:comparison_00014_00024}
\end{figure*}

\begin{figure*}[t]
    \centering
    \includegraphics[width=\linewidth]{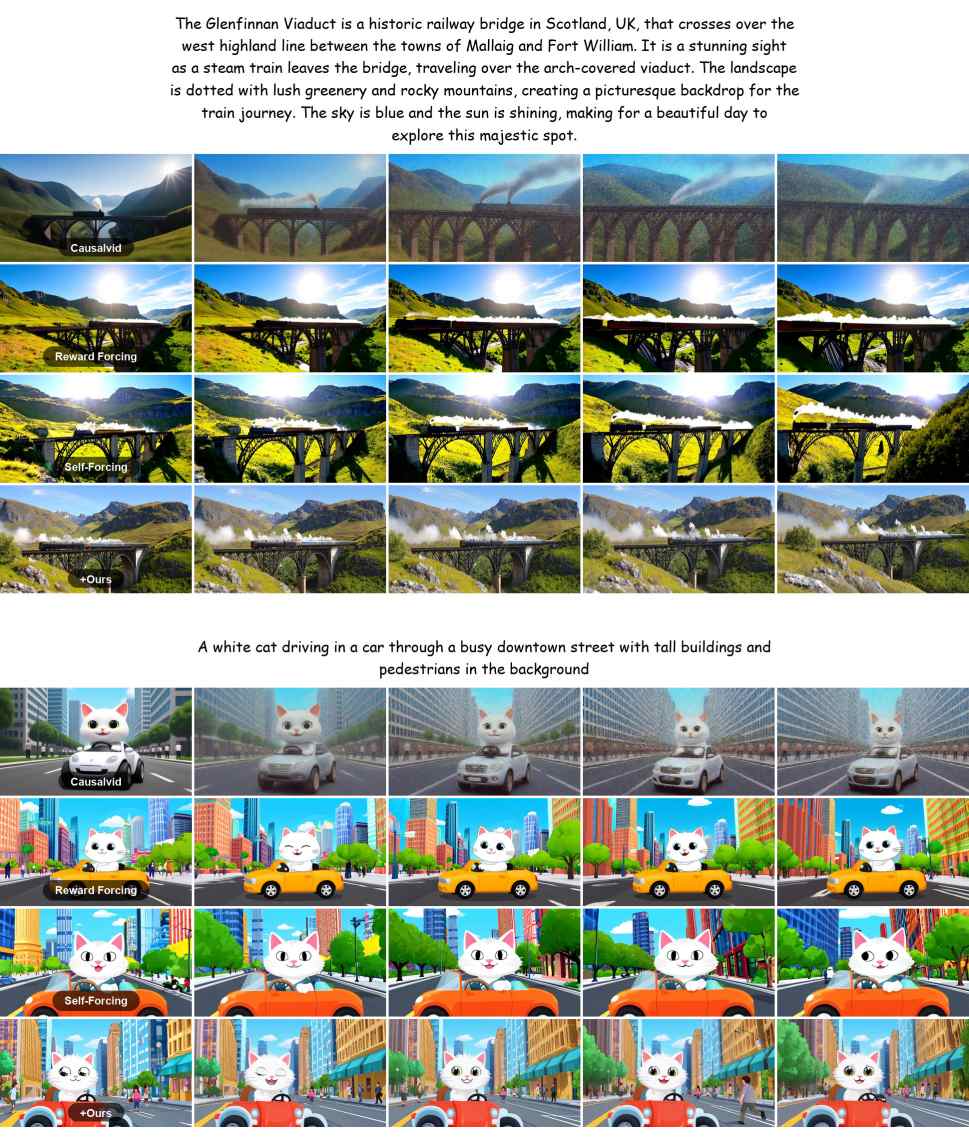}
    \caption{Additional qualitative comparison on long video generation (30s) under the single-prompt setting (Set 3).}
    \label{fig:comparison_00039_00050}
\end{figure*}

\begin{figure*}[t]
    \centering
    \includegraphics[width=\linewidth]{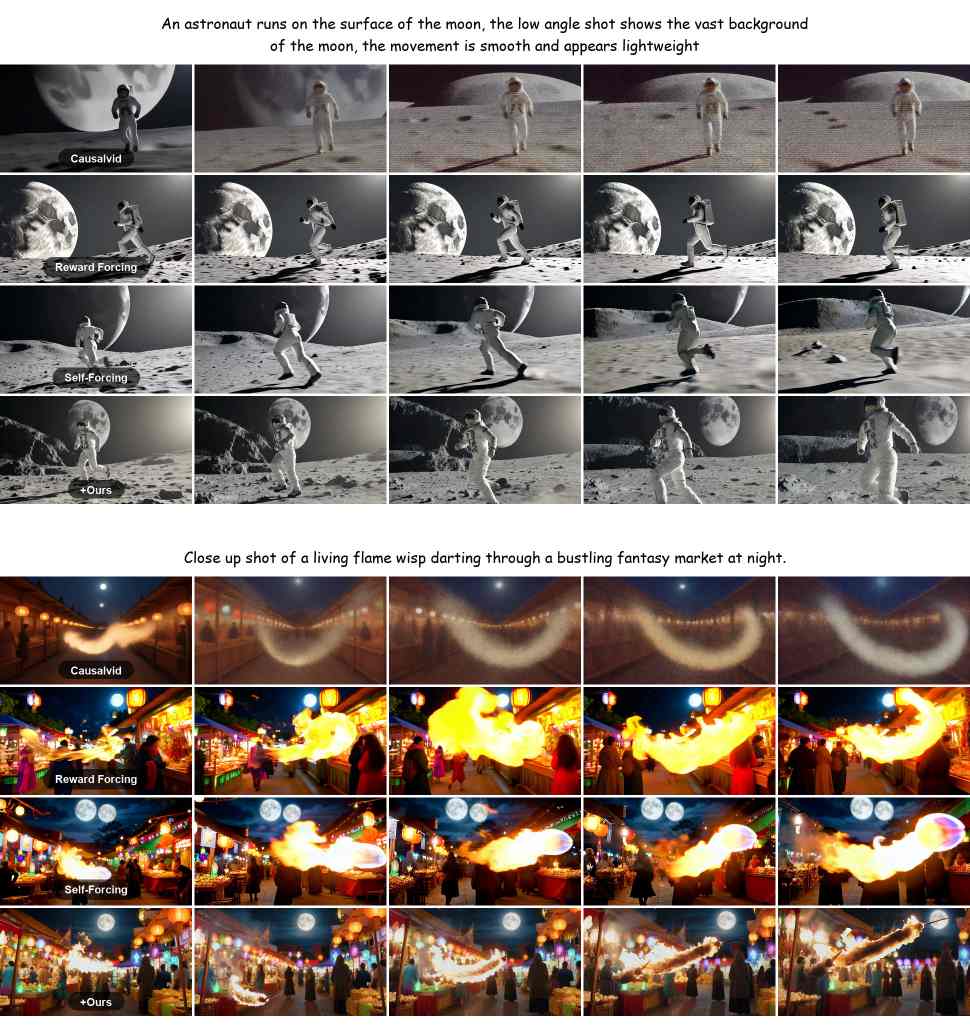}
    \caption{Additional qualitative comparison on long video generation (30s) under the single-prompt setting (Set 4).}
    \label{fig:comparison_00060_00074}
\end{figure*}

\begin{figure*}[t]
    \centering
    \includegraphics[width=\linewidth]{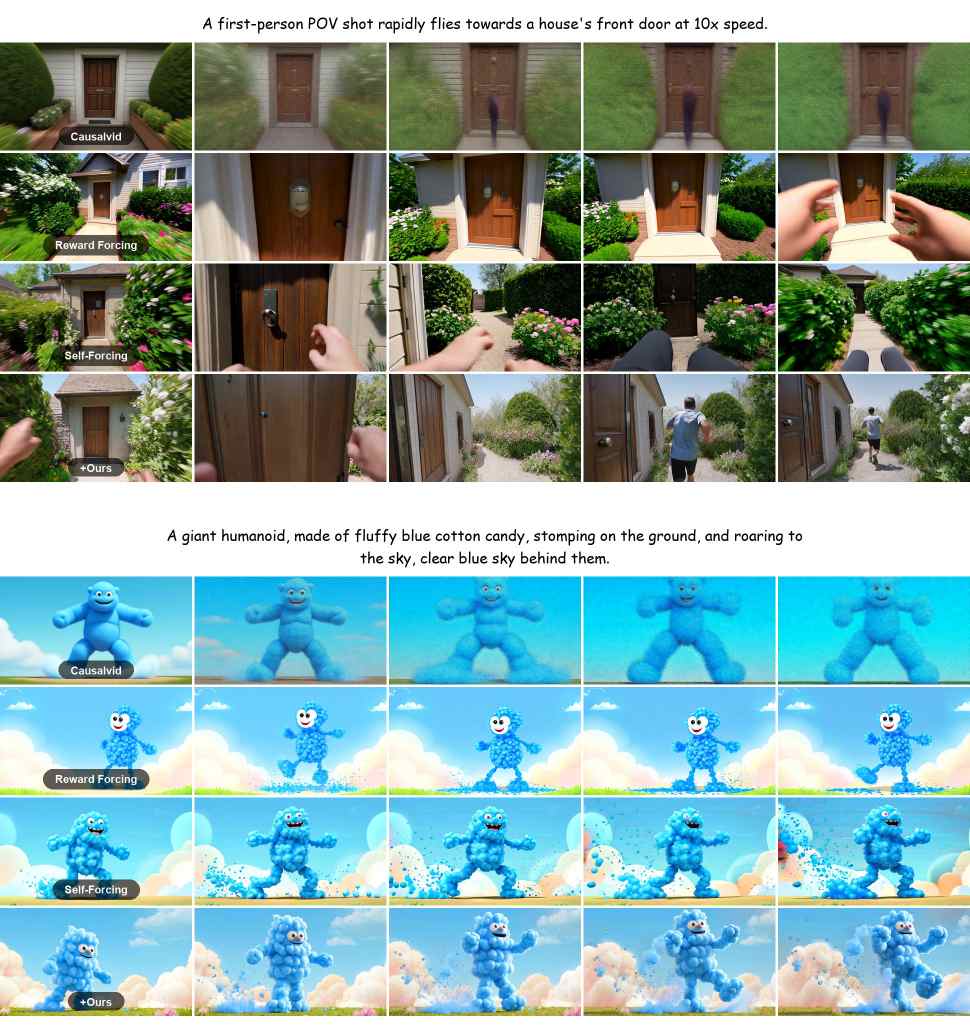}
    \caption{Additional qualitative comparison on long video generation (30s) under the single-prompt setting (Set 5).}
    \label{fig:comparison_00085_00104}
\end{figure*}

\begin{figure*}[t]
    \centering
    \includegraphics[width=\linewidth]{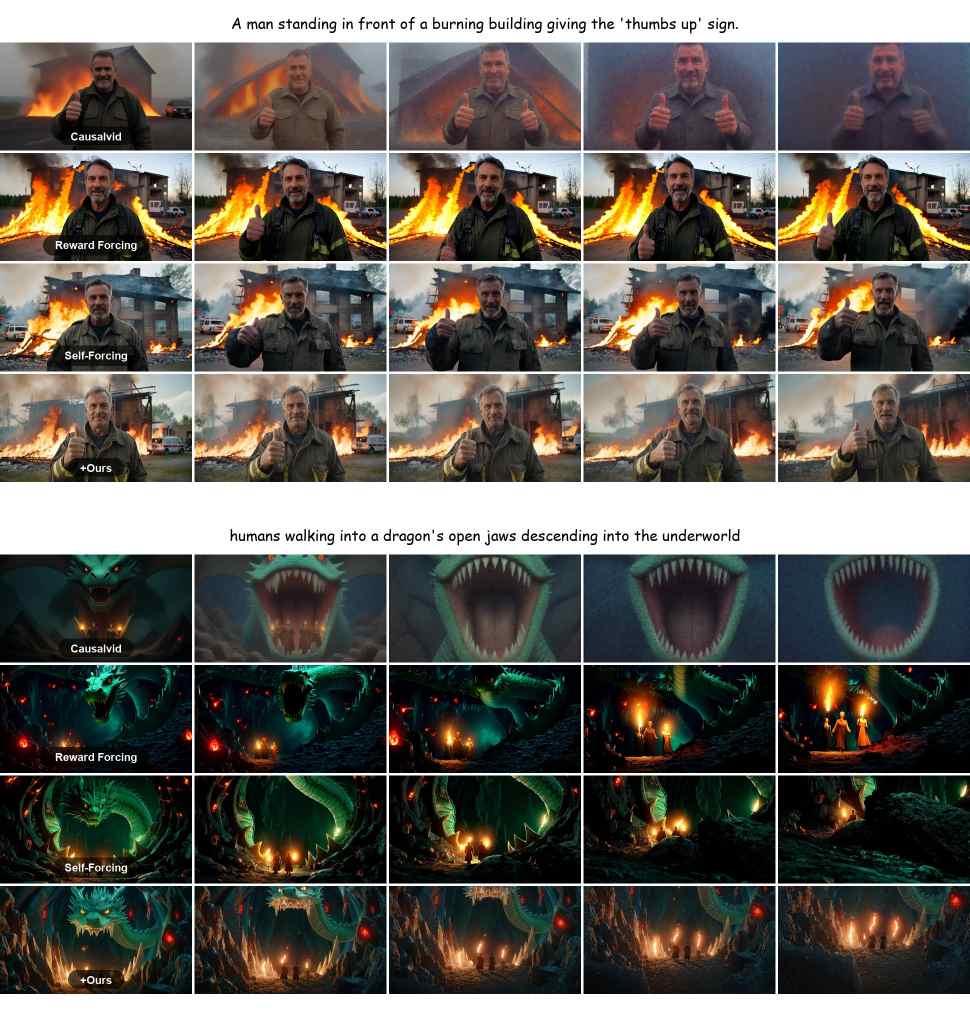}
    \caption{Additional qualitative comparison on long video generation (30s) under the single-prompt setting (Set 6).}
    \label{fig:comparison_00107_00135}
\end{figure*}

\begin{figure*}[t]
    \centering
    \includegraphics[width=\linewidth]{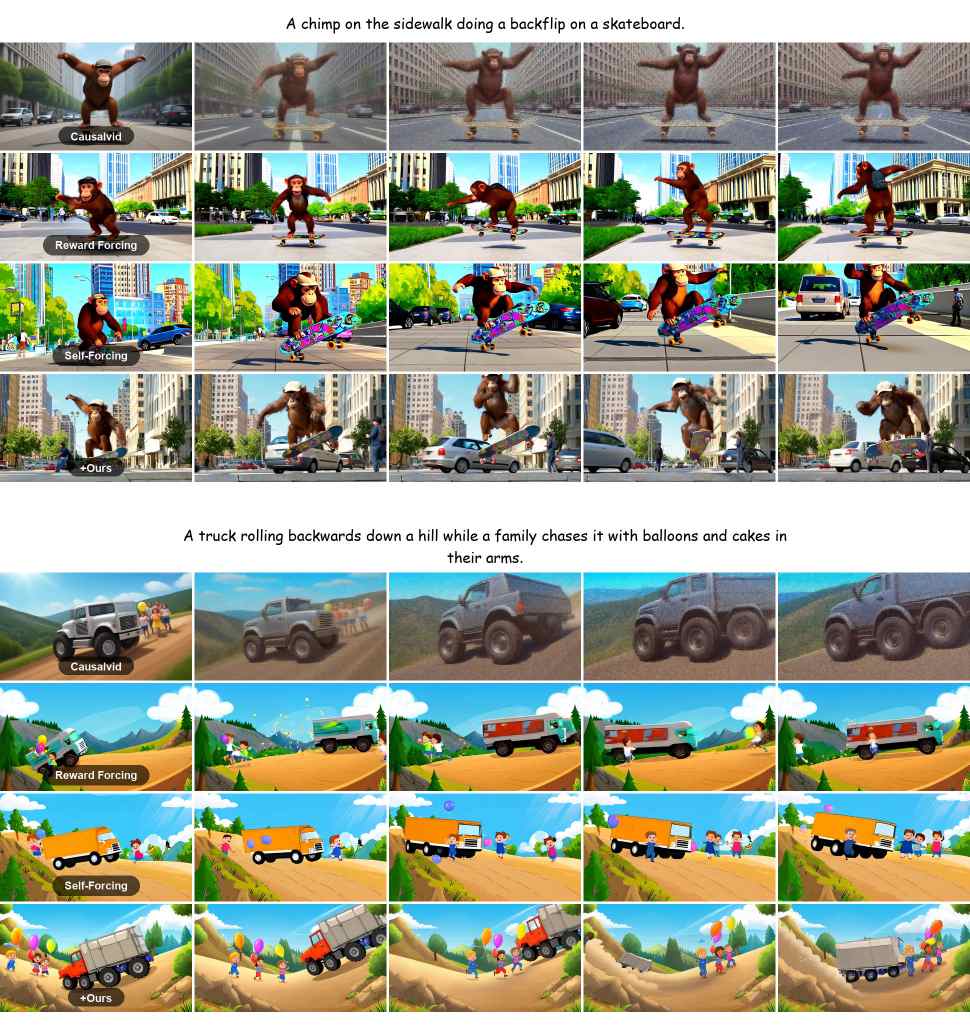}
    \caption{Additional qualitative comparison on long video generation (30s) under the single-prompt setting (Set 7).}
    \label{fig:comparison_00150_00193}
\end{figure*}

\begin{figure*}[t]
    \centering
    \includegraphics[width=\linewidth]{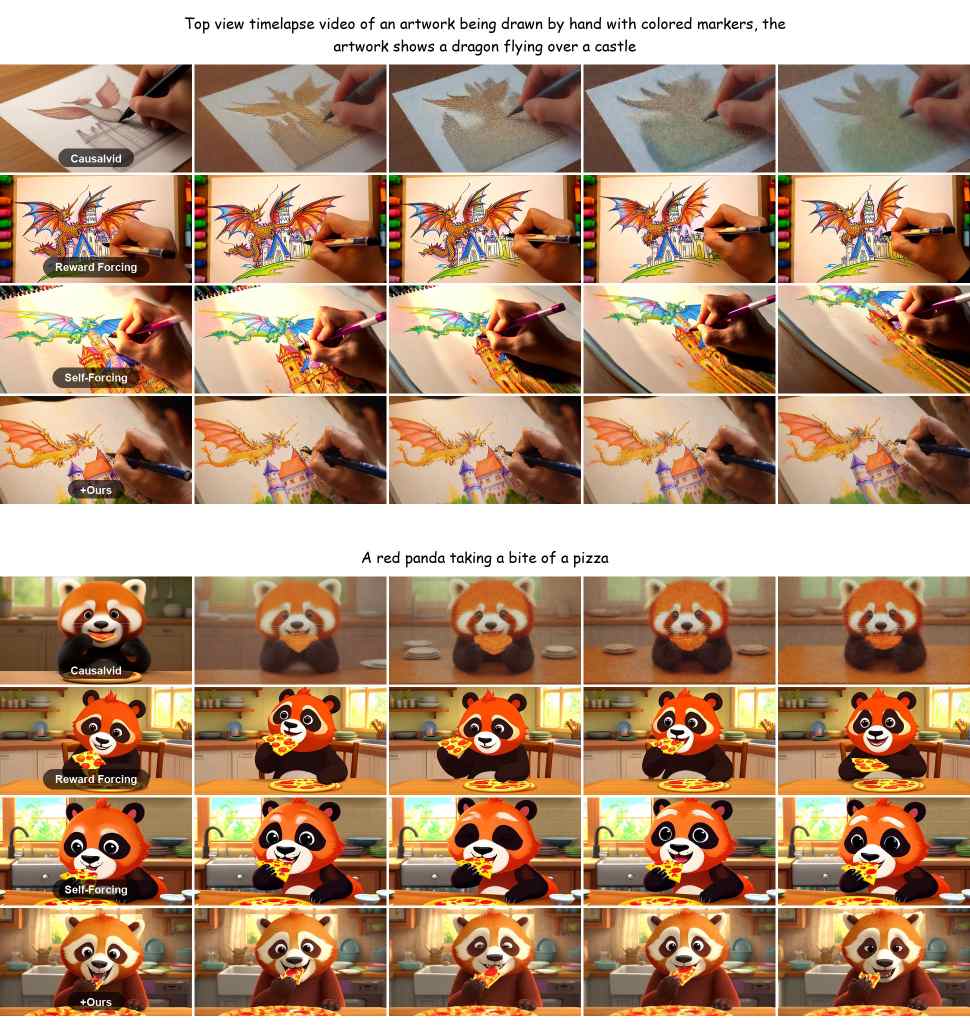}
    \caption{Additional qualitative comparison on long video generation (30s) under the single-prompt setting (Set 8).}
    \label{fig:comparison_00257_00265}
\end{figure*}

\begin{figure*}[t]
    \centering
    \includegraphics[width=\linewidth]{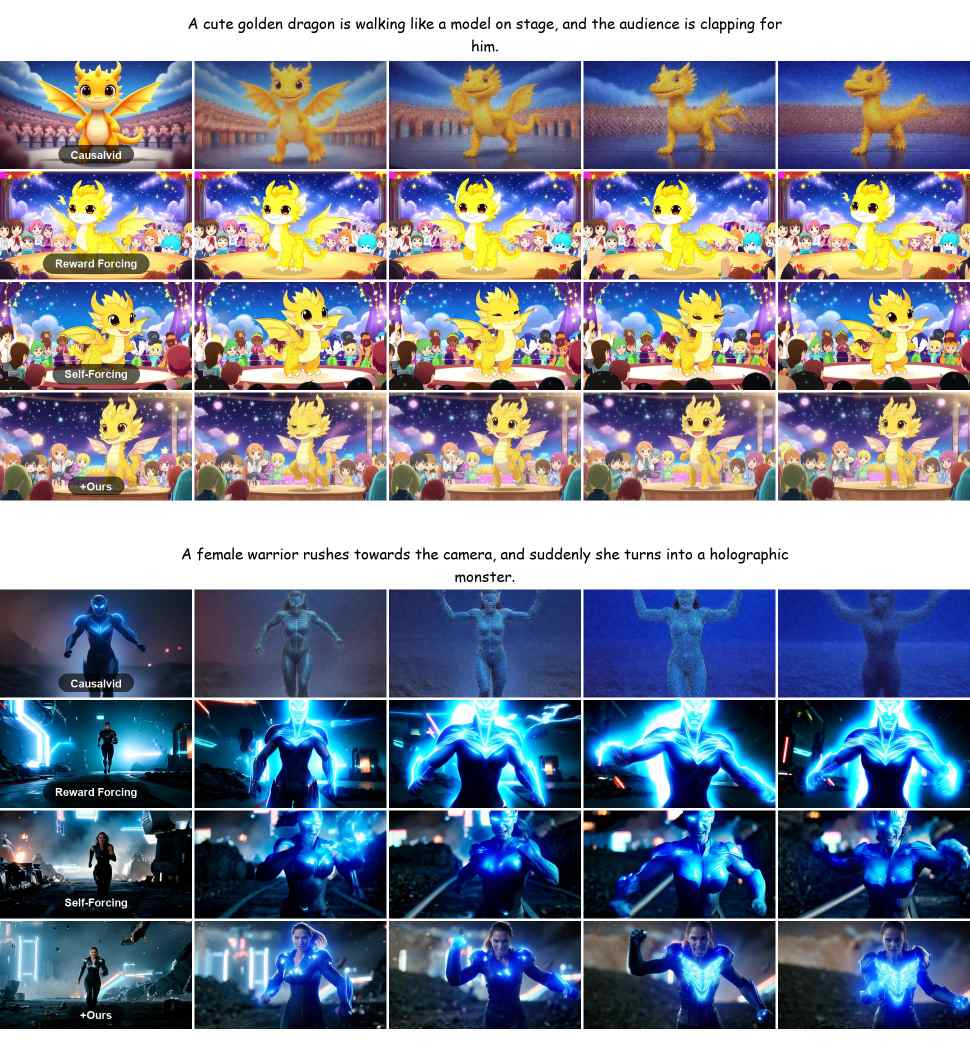}
    \caption{Additional qualitative comparison on long video generation (30s) under the single-prompt setting (Set 9).}
    \label{fig:comparison_00275_00297}
\end{figure*}

\begin{figure*}[t]
    \centering
    \includegraphics[width=\linewidth]{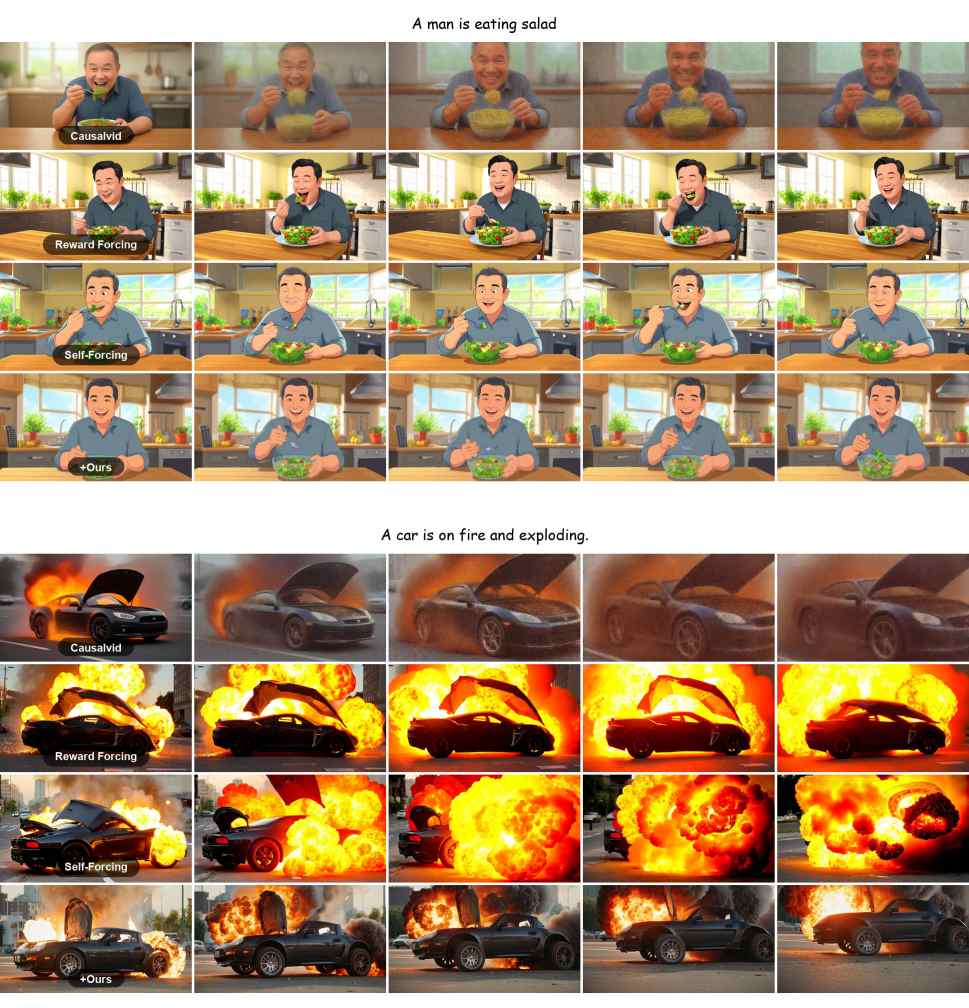}
    \caption{Additional qualitative comparison on long video generation (30s) under the single-prompt setting (Set 10).}
    \label{fig:comparison_00314_00342}
\end{figure*}

\begin{figure*}[t]
    \centering
    \includegraphics[width=\linewidth]{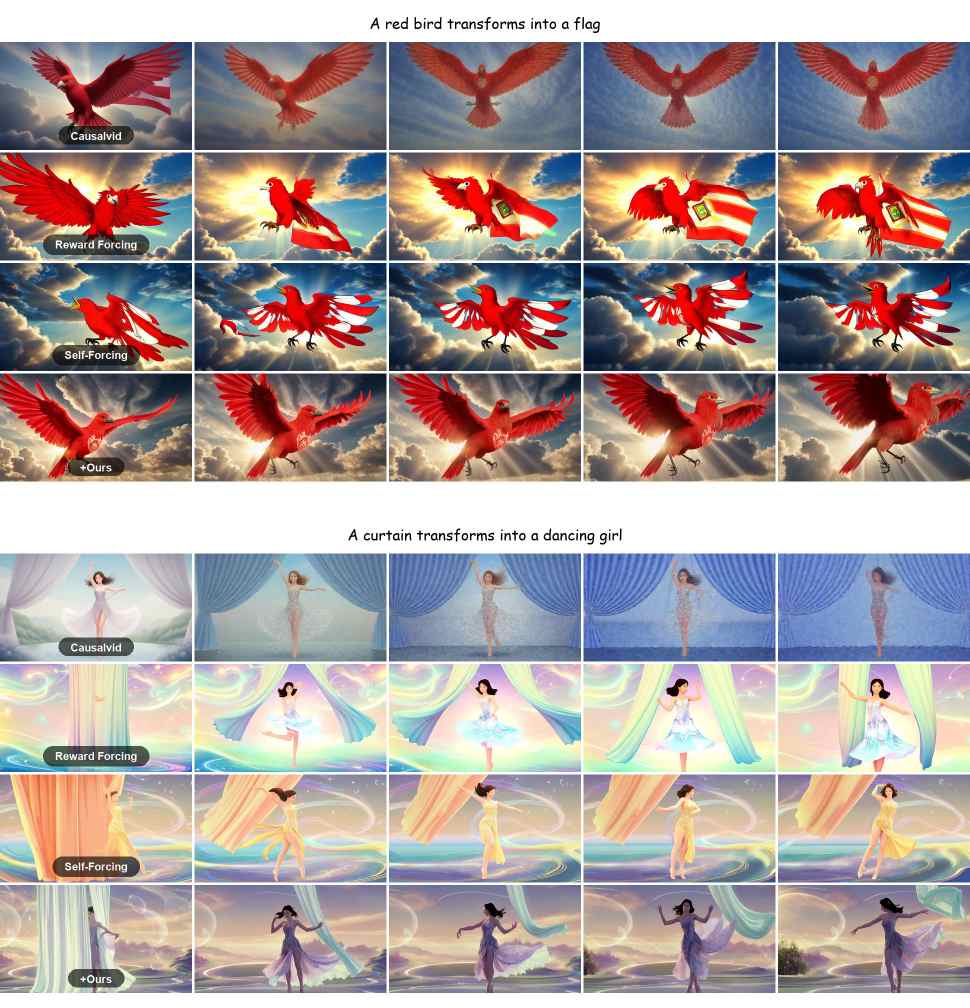}
    \caption{Additional qualitative comparison on long video generation (30s) under the single-prompt setting (Set 11).}
    \label{fig:comparison_00371_00372}
\end{figure*}

\begin{figure*}[t]
    \centering
    \includegraphics[width=\linewidth]{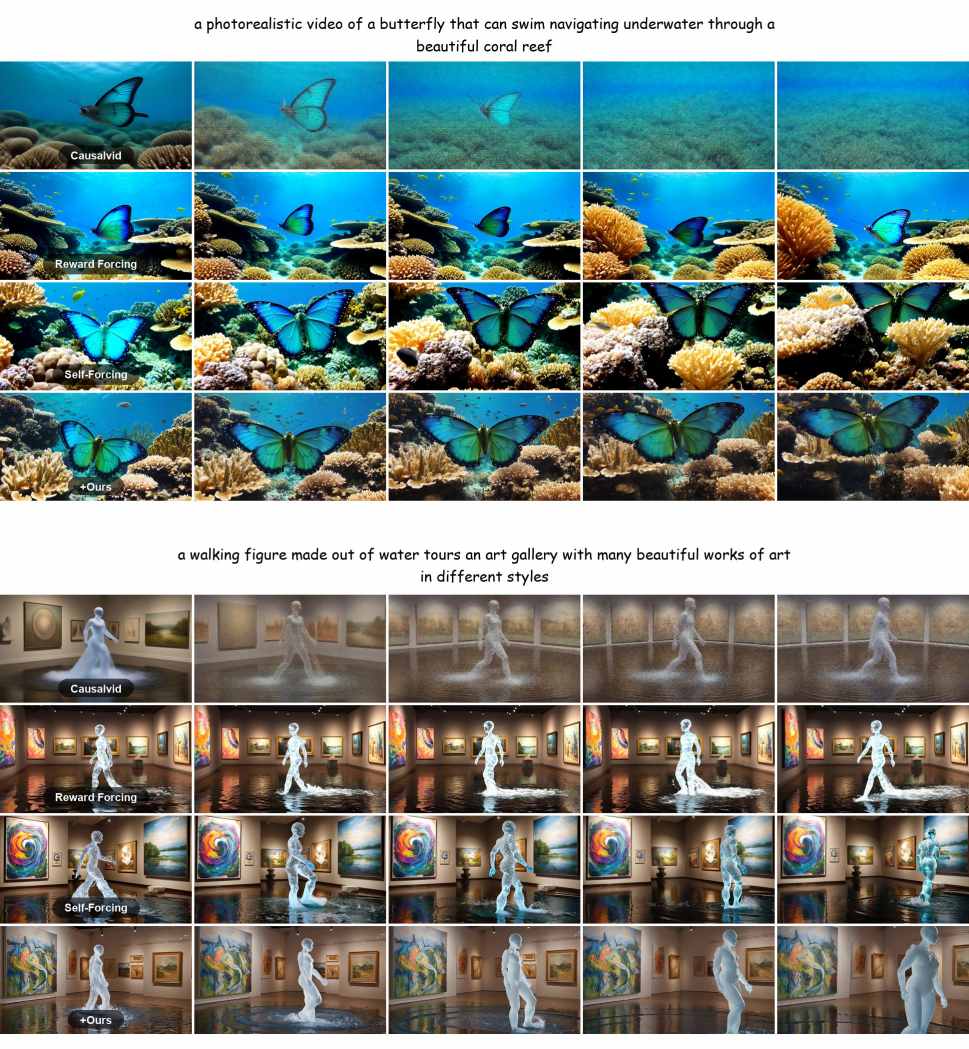}
    \caption{Additional qualitative comparison on long video generation (30s) under the single-prompt setting (Set 12).}
    \label{fig:comparison_00383_00386}
\end{figure*}

\begin{figure*}[t]
    \centering
    \includegraphics[width=\linewidth]{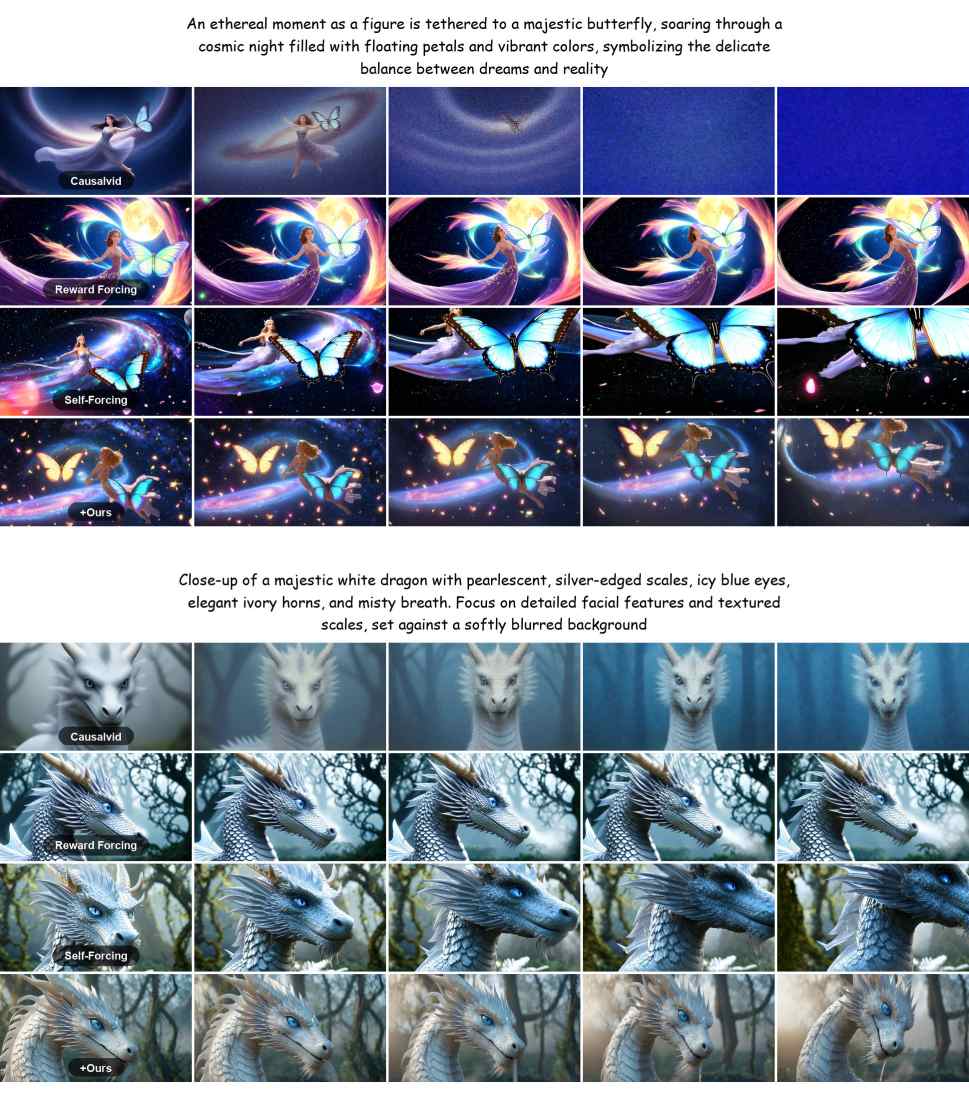}
    \caption{Additional qualitative comparison on long video generation (30s) under the single-prompt setting (Set 13).}
    \label{fig:comparison_00387_00391}
\end{figure*}

\begin{figure*}[t]
    \centering
    \includegraphics[width=\linewidth]{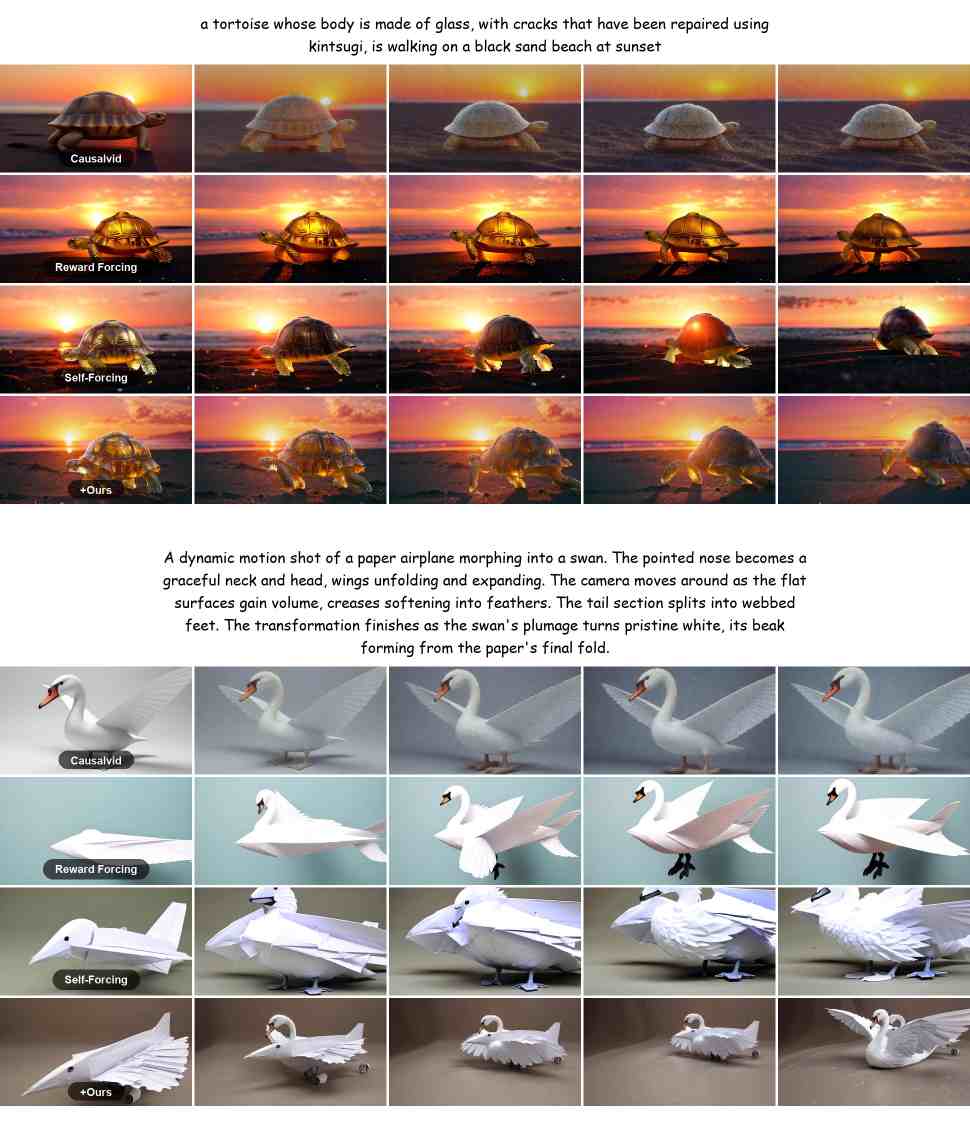}
    \caption{Additional qualitative comparison on long video generation (30s) under the single-prompt setting (Set 14).}
    \label{fig:comparison_00400_00410}
\end{figure*}

\begin{figure*}[t]
    \centering
    \includegraphics[width=\linewidth]{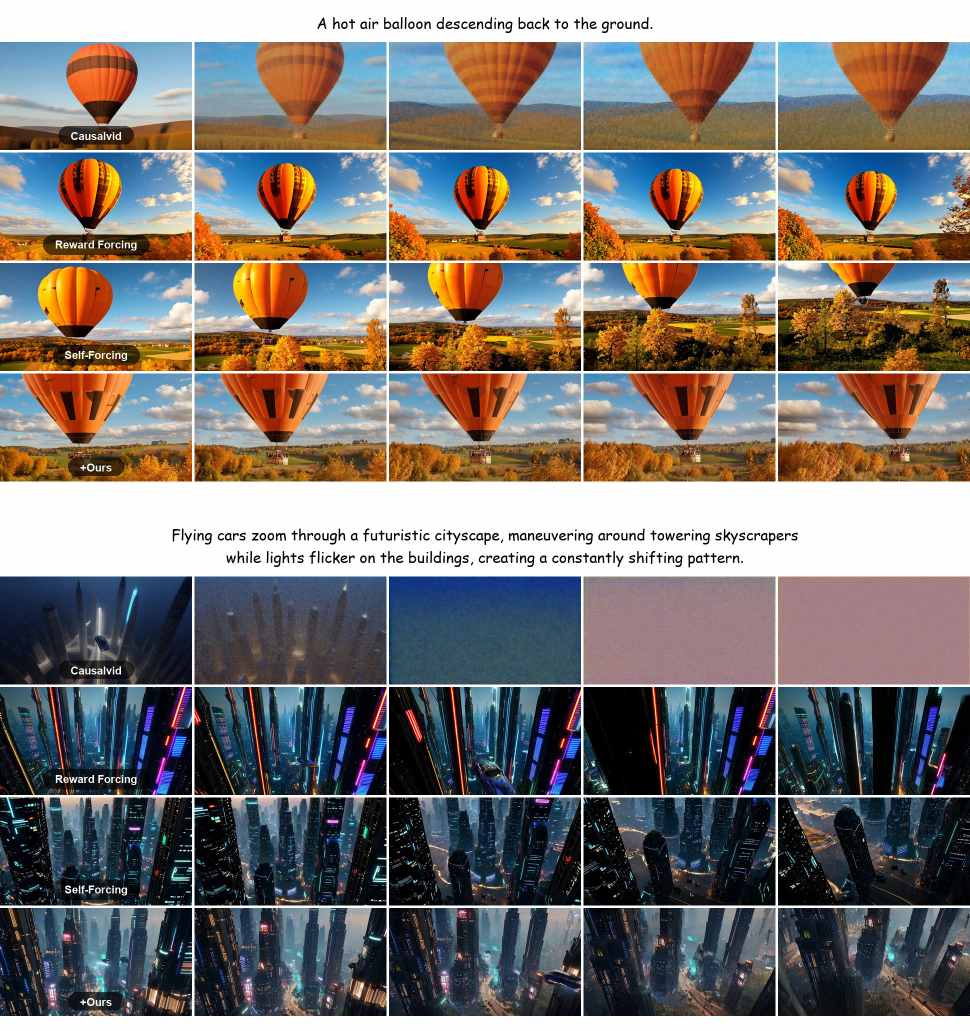}
    \caption{Additional qualitative comparison on long video generation (30s) under the single-prompt setting (Set 15).}
    \label{fig:comparison_00484_00491}
\end{figure*}

\begin{figure*}[t]
    \centering
    \includegraphics[width=\linewidth]{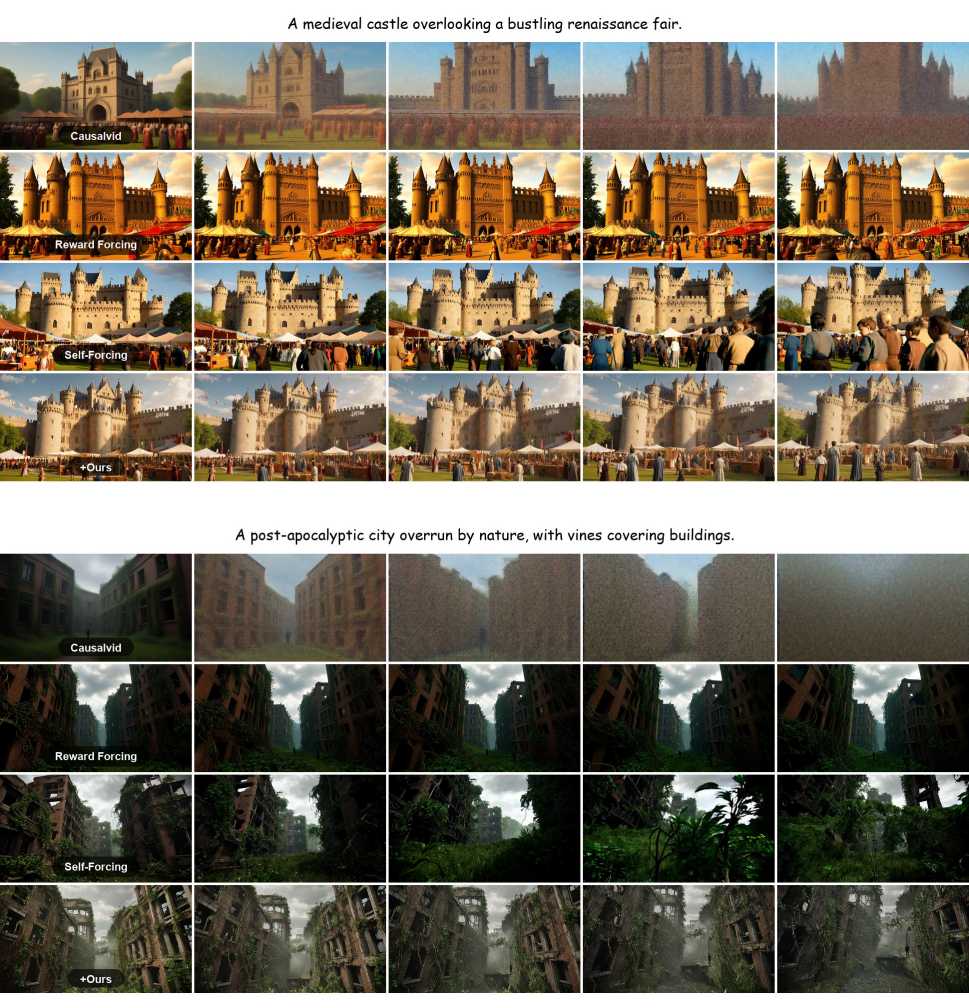}
    \caption{Additional qualitative comparison on long video generation (30s) under the single-prompt setting (Set 16).}
    \label{fig:comparison_00500_00507}
\end{figure*}

\begin{figure*}[t]
    \centering
    \includegraphics[width=\linewidth]{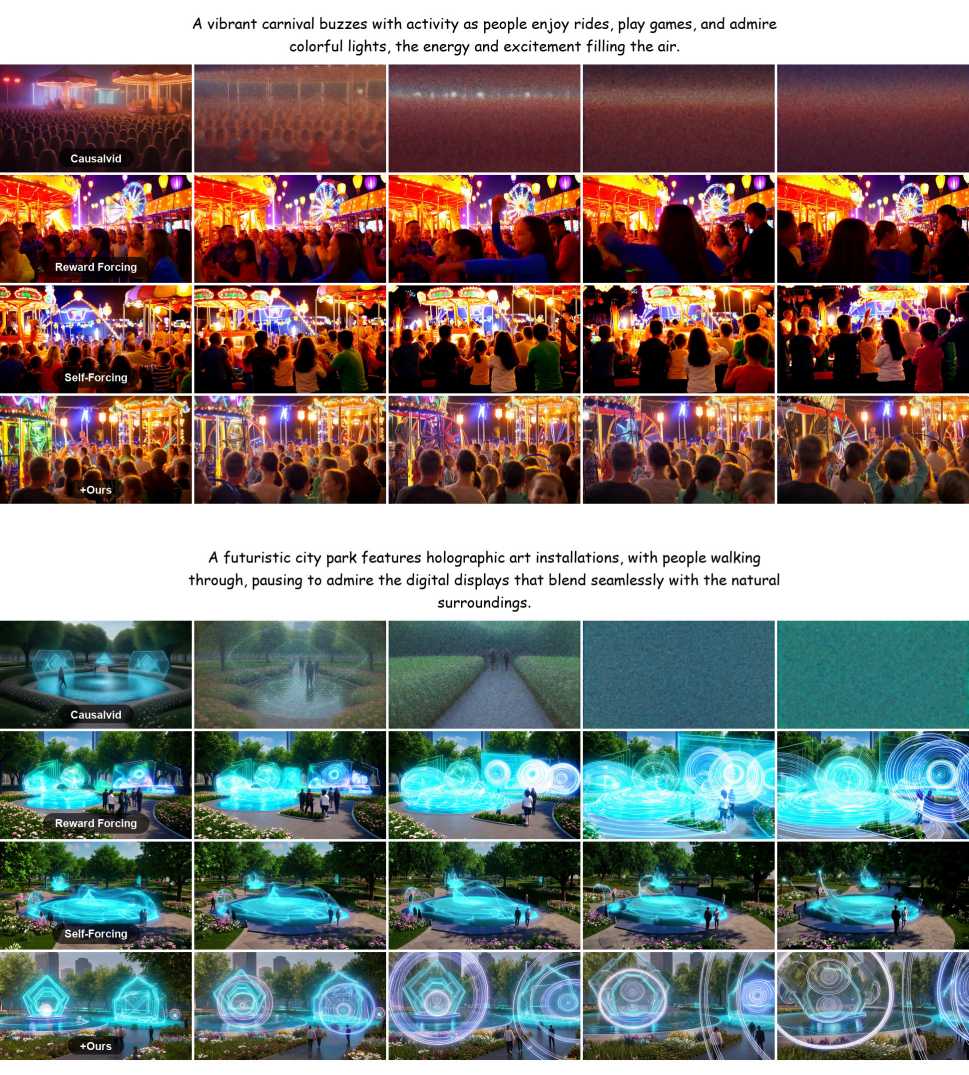}
    \caption{Additional ualitative comparison on long video generation (30s) under the single-prompt setting (Set 17).}
    \label{fig:comparison_00538_00540}
\end{figure*}

\begin{figure*}[t]
    \centering
    \includegraphics[width=\linewidth]{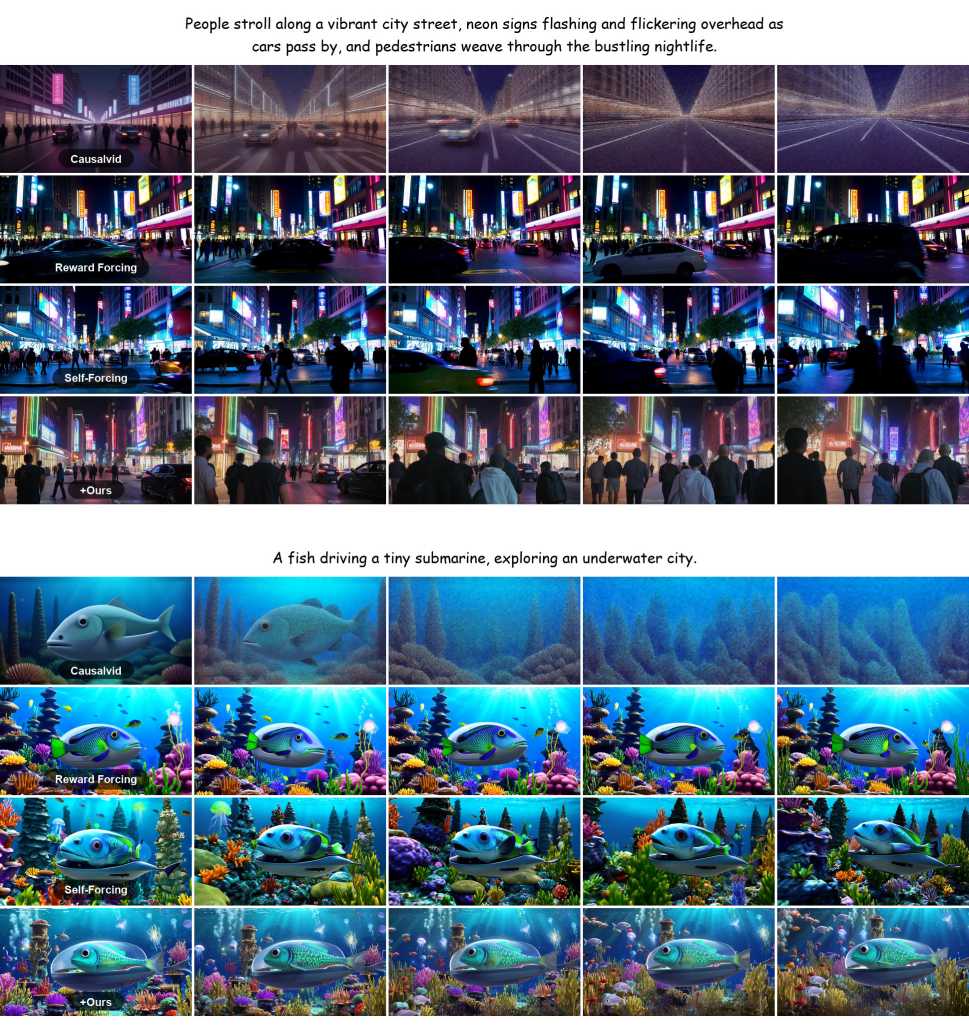}
    \caption{Additional ualitative comparison on long video generation (30s) under the single-prompt setting (Set 18).}
    \label{fig:comparison_00546_00565}
\end{figure*}

\begin{figure*}[t]
    \centering
    \includegraphics[width=\linewidth]{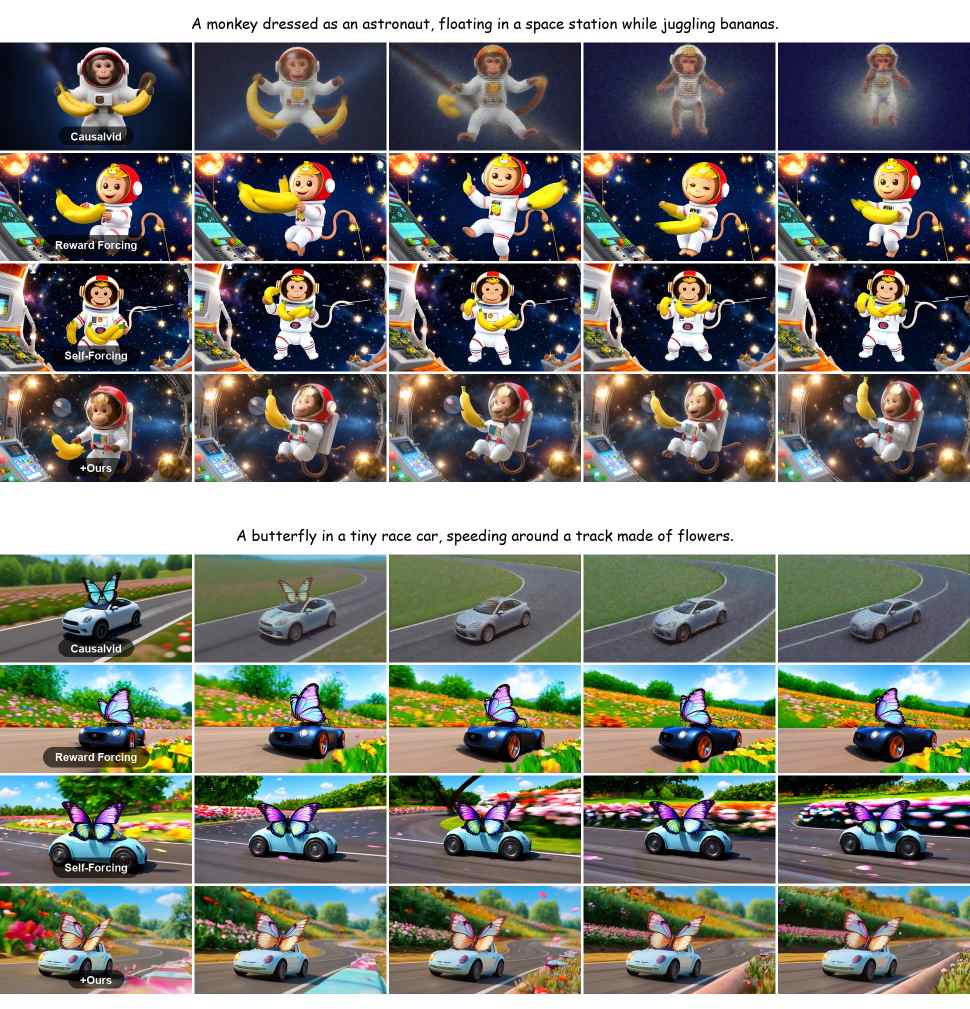}
    \caption{Additional ualitative comparison on long video generation (30s) under the single-prompt setting (Set 19).}
    \label{fig:comparison_00567_00575}
\end{figure*}

\begin{figure*}[t]
    \centering
    \includegraphics[width=\linewidth]{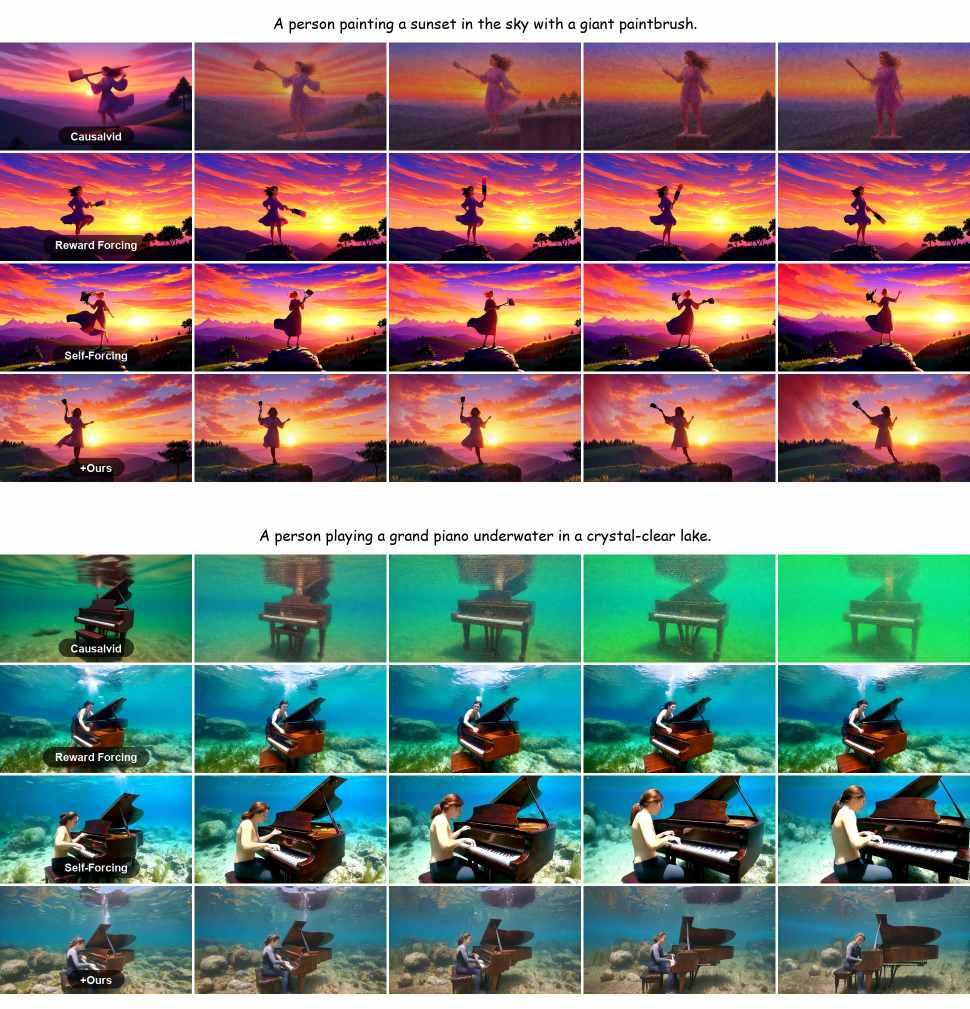}
    \caption{Additional qualitative comparison on long video generation (30s) under the single-prompt setting (Set 20).}
    \label{fig:comparison_00594_00596}
\end{figure*}

\begin{figure*}[t]
    \centering
    \includegraphics[width=\linewidth]{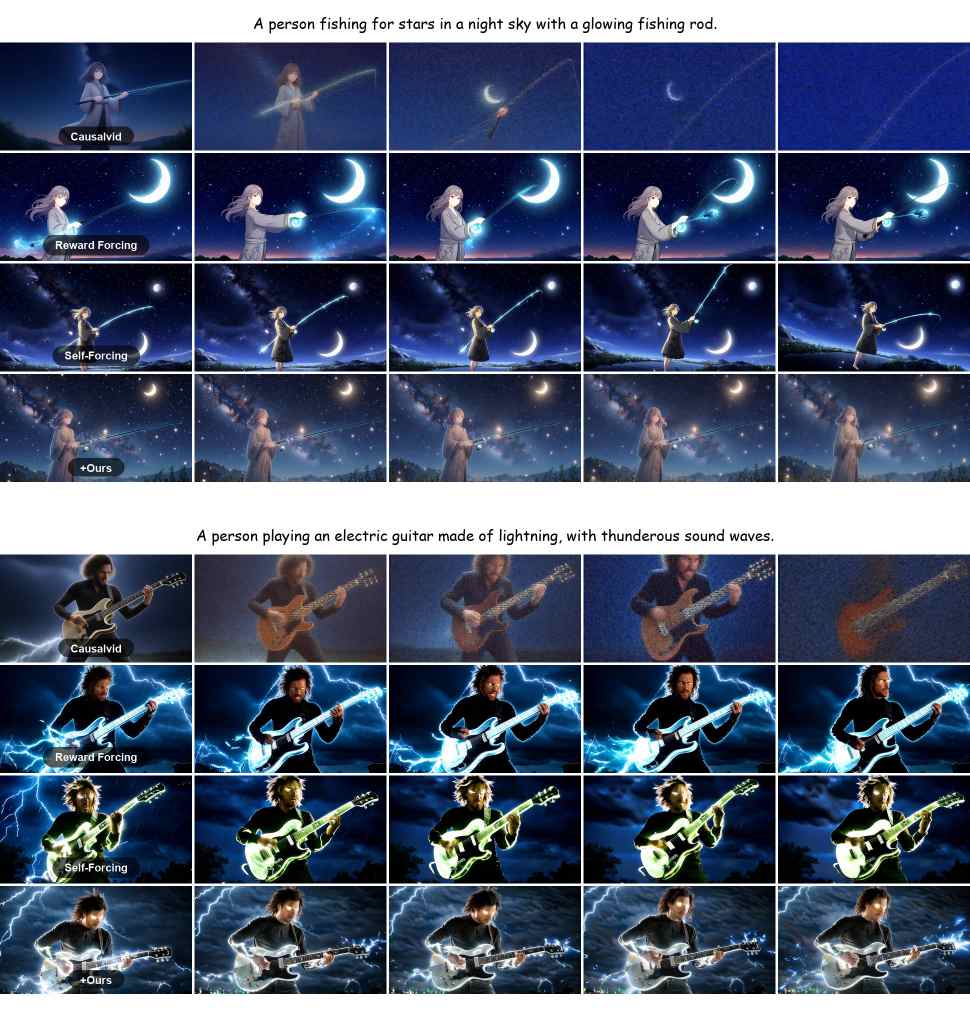}
    \caption{Additional qualitative comparison on long video generation (30s) under the single-prompt setting (Set 21).}
    \label{fig:comparison_00607_00617}
\end{figure*}

\begin{figure*}[t]
    \centering
    \includegraphics[width=\linewidth]{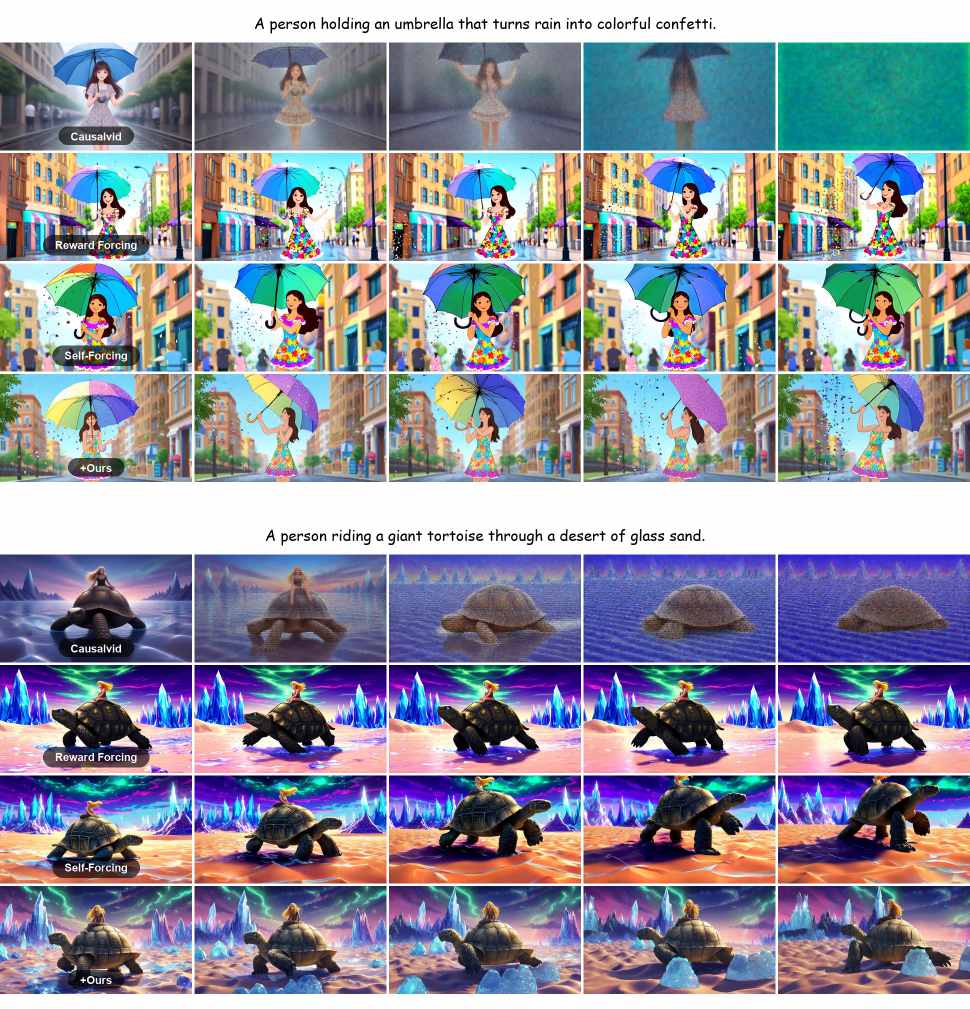}
    \caption{Additional qualitative comparison on long video generation (30s) under the single-prompt setting (Set 22).}
    \label{fig:comparison_00630_00635}
\end{figure*}

\begin{figure*}[t]
    \centering
    \includegraphics[width=\linewidth]{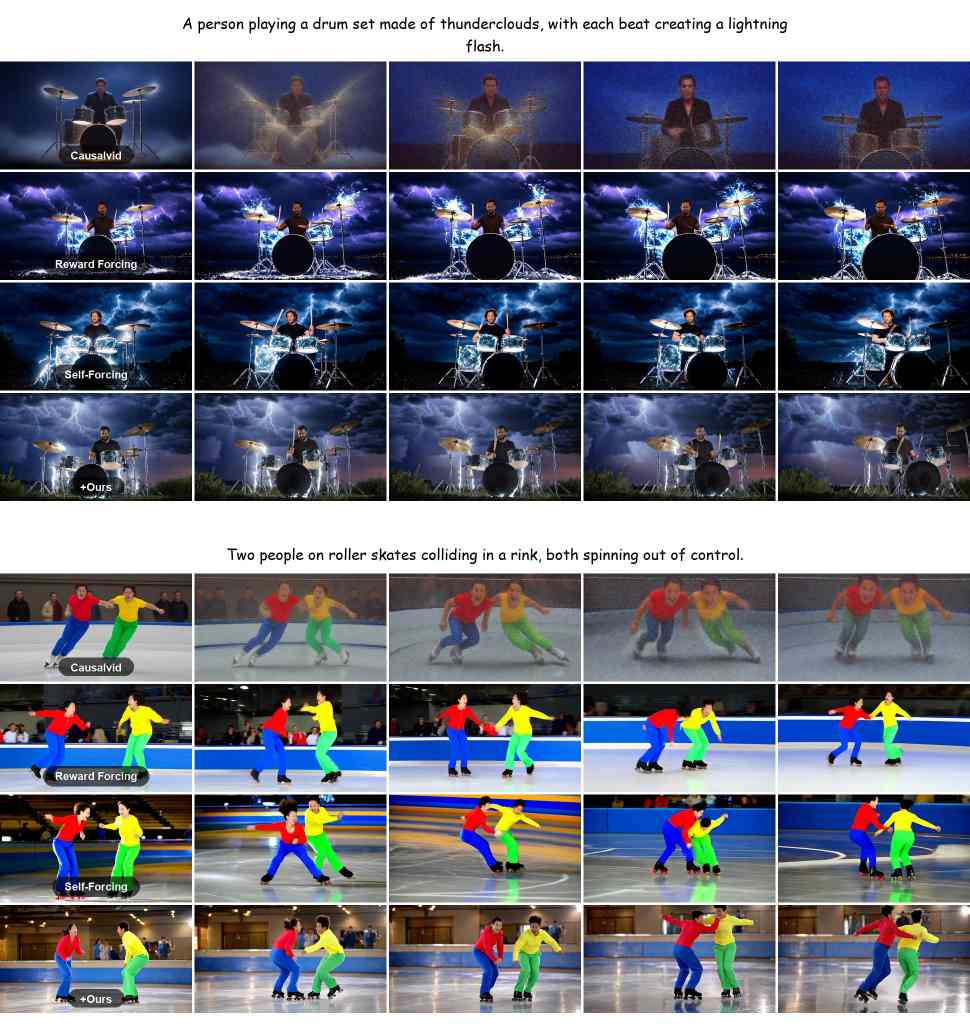}
    \caption{Additional qualitative comparison on long video generation (30s) under the single-prompt setting (Set 23).}
    \label{fig:comparison_00636_00650}
\end{figure*}

\begin{figure*}[t]
    \centering
    \includegraphics[width=\linewidth]{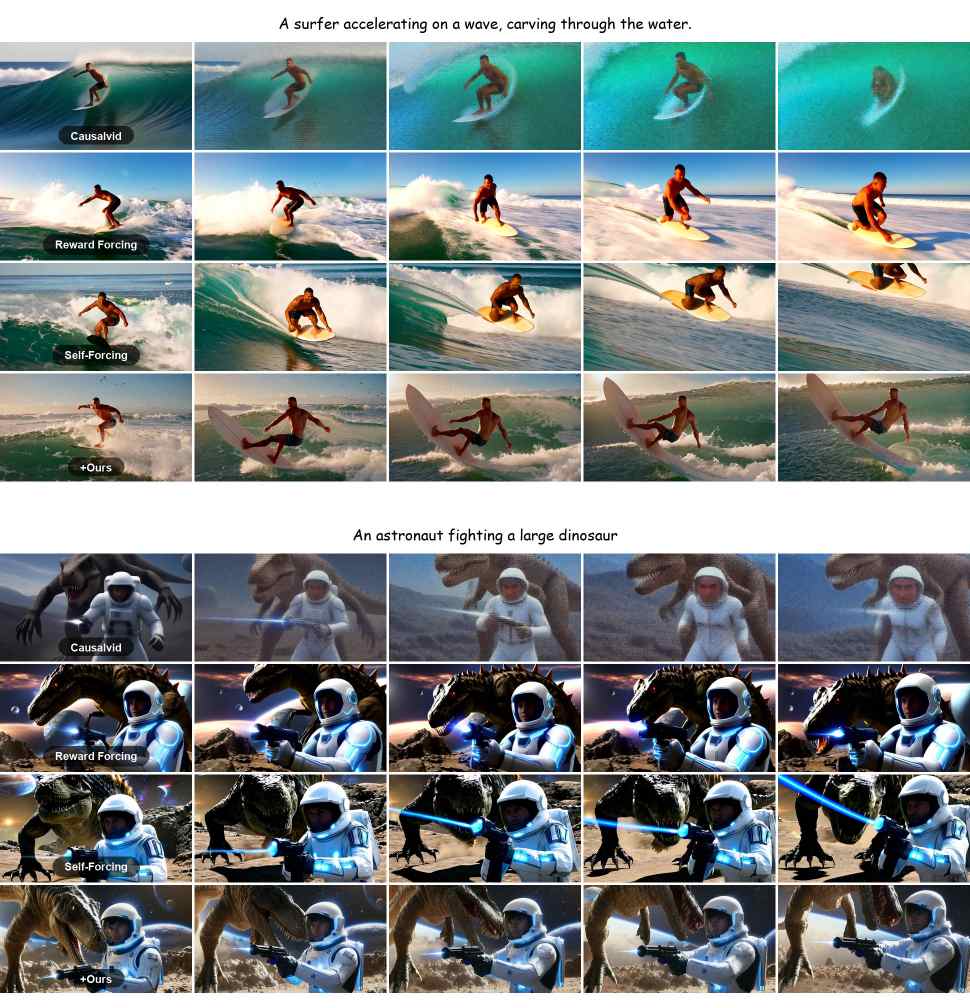}
    \caption{Additional qualitative comparison on long video generation (30s) under the single-prompt setting (Set 24).}
    \label{fig:comparison_00679_00690}
\end{figure*}

\begin{figure*}[t]
    \centering
    \includegraphics[width=\linewidth]{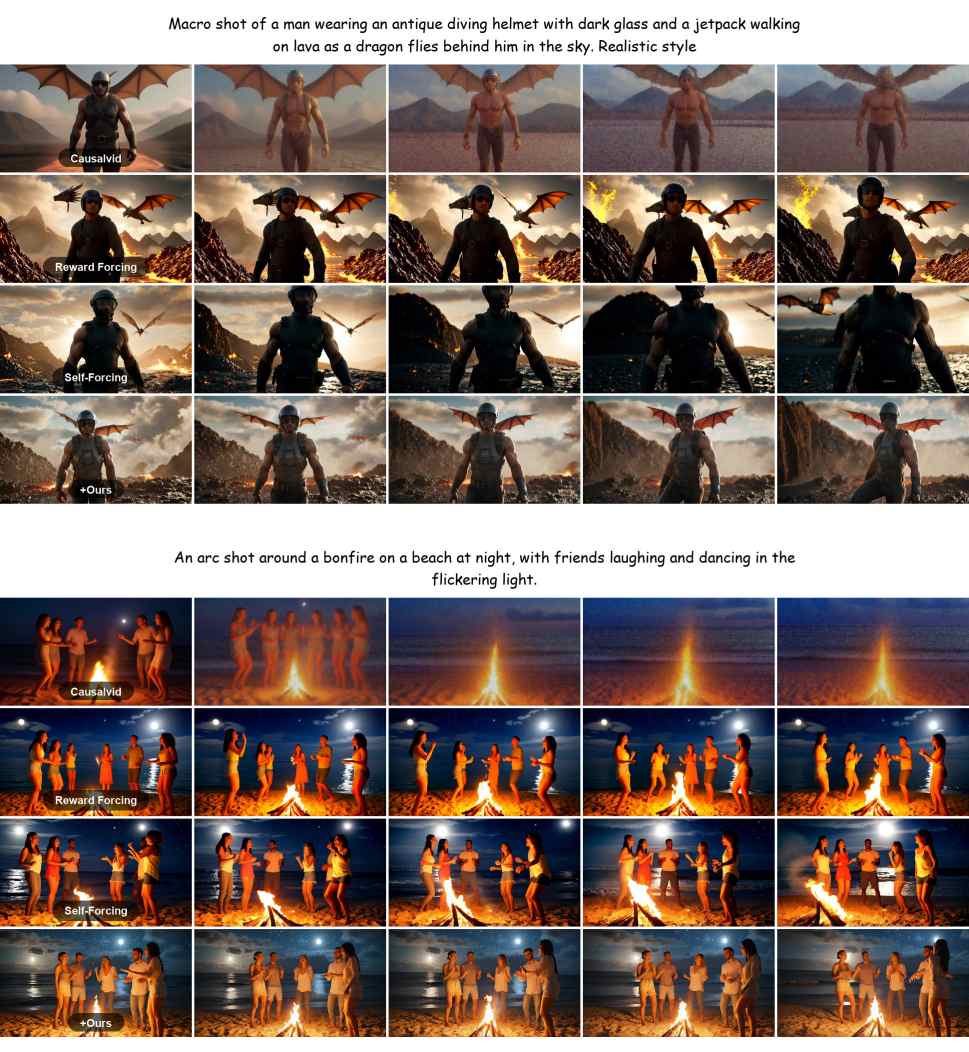}
    \caption{Additional qualitative comparison on long video generation (30s) under the single-prompt setting (Set 25).}
    \label{fig:comparison_00692_00753}
\end{figure*}

\begin{figure*}[t]
    \centering
    \includegraphics[width=\linewidth]{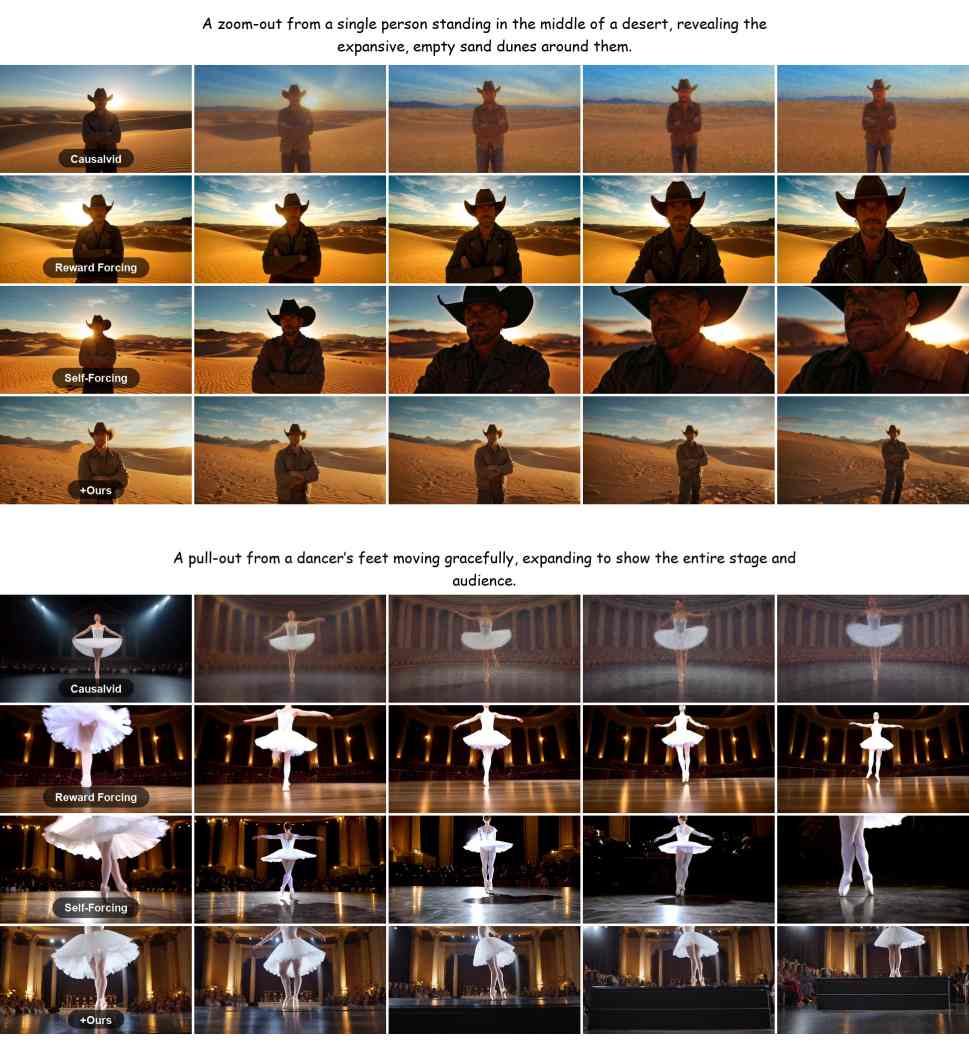}
    \caption{Additional qualitative comparison on long video generation (30s) under the single-prompt setting (Set 26).}
    \label{fig:comparison_00836_00843}
\end{figure*}

\begin{figure*}[t]
    \centering
    \includegraphics[width=\linewidth]{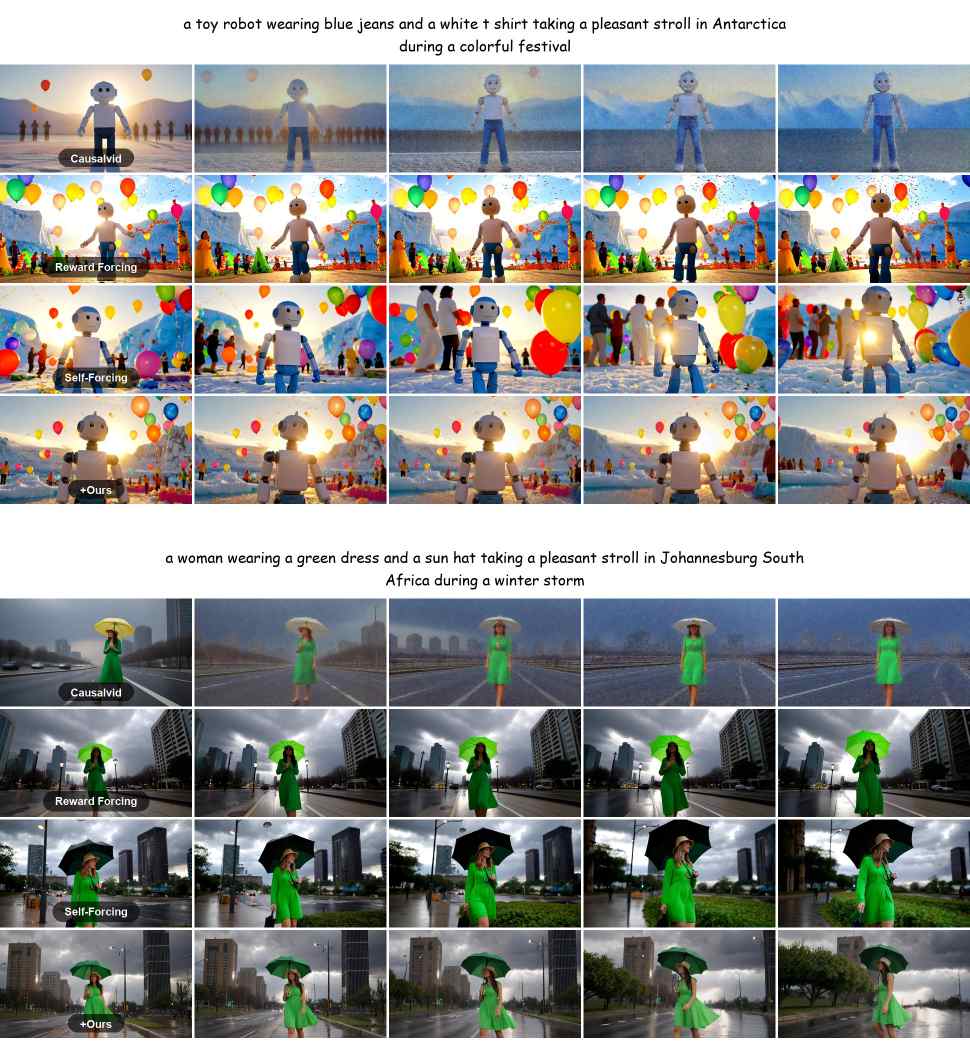}
    \caption{Additional qualitative comparison on long video generation (30s) under the single-prompt setting (Set 27).}
    \label{fig:comparison_00902_00945}
\end{figure*}

\begin{figure*}[t]
    \centering
    \includegraphics[width=\linewidth]{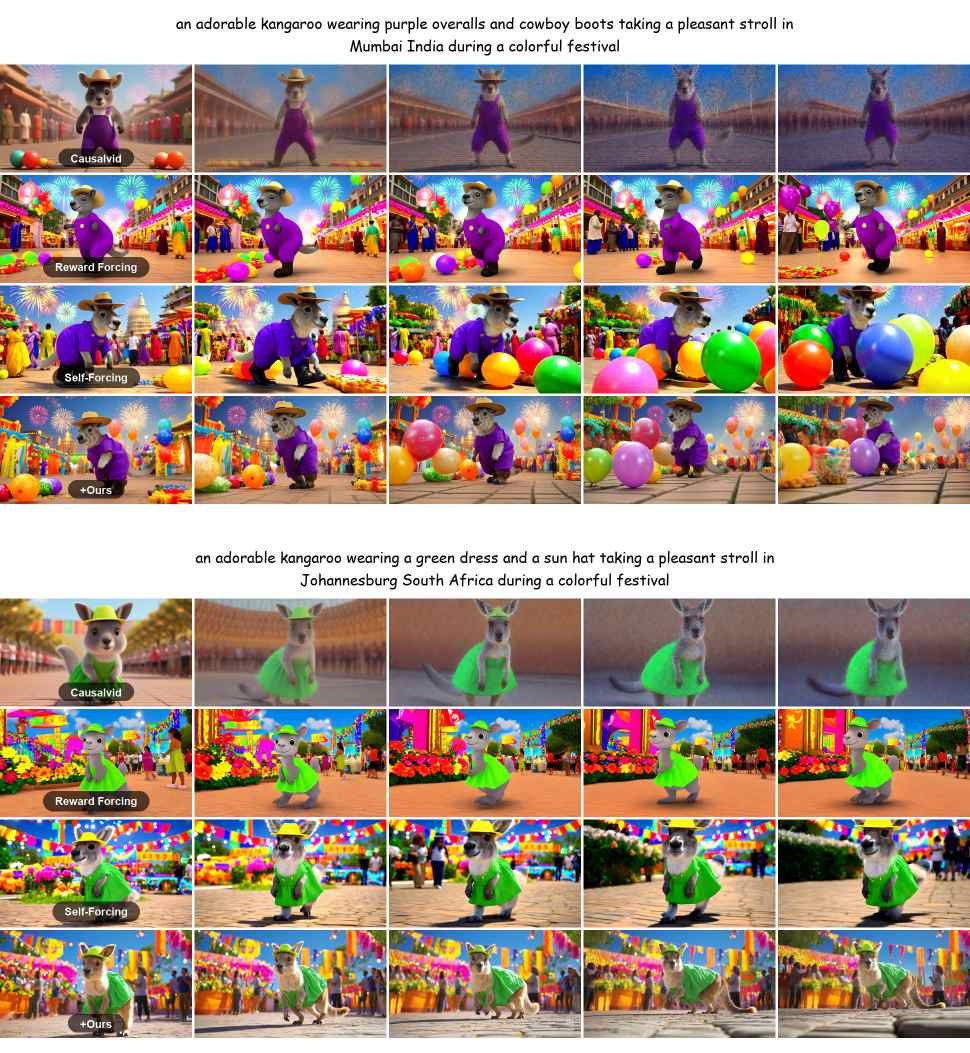}
    \caption{Additional qualitative comparison on long video generation (30s) under the single-prompt setting (Set 28).}
    \label{fig:comparison_00959_00971}
\end{figure*}

\end{document}